\documentclass[letterpaper]{article} 
\usepackage{aaai23}  
\usepackage{times}  
\usepackage{helvet}  
\usepackage{courier}  
\usepackage[hyphens]{url}  
\usepackage{graphicx} 
\usepackage{natbib}  
\usepackage{caption} 
\usepackage{algorithm}
\usepackage{amsmath,amssymb,amsthm,bm}
\usepackage{algpseudocode}
\usepackage{multirow}
\usepackage{booktabs}
\usepackage[switch]{lineno} %
\usepackage{enumerate}
\usepackage{color}
\usepackage[american]{babel}

\newtheorem{theorem}{Theorem}
\newtheorem{lemma}{Lemma}
\newtheorem{corollary}{Corollary}
\newtheorem{remark}{Remark}
\usepackage{newfloat}
\usepackage{listings}
\DeclareCaptionStyle{ruled}{labelfont=normalfont,labelsep=colon,strut=off} 
\floatstyle{ruled}
\newfloat{listing}{tb}{lst}{}
\floatname{listing}{Listing}
\title{Improved Kernel Alignment Regret Bound for Online Kernel Learning}
\author{
    Junfan Li, Shizhong Liao\thanks{Corresponding author.}\\
}
\affiliations{

    College of Intelligence and Computing,
    Tianjin University, Tianjin 300350, China\\
    \{junfli,szliao\}@tju.edu.cn
}

\usepackage{bibentry}

\begin{document}

\maketitle

\begin{abstract}
  In this paper,
  we improve the kernel alignment regret bound for online kernel learning
  in the regime of the Hinge loss function.
  Previous algorithm achieves
  a regret of $O((\mathcal{A}_TT\ln{T})^{\frac{1}{4}})$
  at a computational complexity (space and per-round time)
  of $O(\sqrt{\mathcal{A}_TT\ln{T}})$,
  where $\mathcal{A}_T$ is called \textit{kernel alignment}.
  We propose an algorithm
  whose regret bound and computational complexity are better than previous results.
  Our results depend on the decay rate of eigenvalues of the kernel matrix.
  If the eigenvalues of the kernel matrix decay exponentially,
  then our algorithm enjoys a regret of $O(\sqrt{\mathcal{A}_T})$
  at a computational complexity of $O(\ln^2{T})$.
  Otherwise,
  our algorithm enjoys a regret of $O((\mathcal{A}_TT)^{\frac{1}{4}})$
  at a computational complexity of $O(\sqrt{\mathcal{A}_TT})$.
  We extend our algorithm to batch learning
  and obtain a $O(\frac{1}{T}\sqrt{\mathbb{E}[\mathcal{A}_T]})$ excess risk bound
  which improves the previous $O(1/\sqrt{T})$ bound.
\end{abstract}

\section{Introduction}

    Online kernel learning is a popular non-parametric method
    for solving large-scale batch learning and online learning problems.
    Online kernel learning algorithms only pass the data once
    and thus are computationally efficient.
    Specifically,
    $\forall t=1,\ldots,T$,
    an online learning algorithm receives an instance $\mathbf{x}_t\in\mathbb{R}^d$.
    Then it selects a hypothesis $f_t\in \mathbb{H}$ and makes prediction $f_t(\mathbf{x}_t)$.
    $\mathbb{H}=\{f\in\mathcal{H}\vert \Vert f\Vert_{\mathcal{H}}\leq U\}$,
    where $\mathcal{H}$ is a reproducing kernel Hilbert space (RKHS).
    After that, the algorithm observes $y_t$
    and suffers a loss $\ell(f_t(\mathbf{x}_t),y_t)$.
    In an online learning setting,
    $\{(\mathbf{x}_t,y_t)^T_{t=1}\}$ may not be i.i.d.,
    and can even be adversarial.
    We focus on the online prediction performance,
    and aim to minimize the cumulative losses $\sum^T_{t=1}\ell(f_t(\mathbf{x}_t),y_t)$.
    We usually use the \textit{regret} to measure the performance,
    which is defined as follows,
    \begin{equation}
    \label{eq:AAAI23:definition_regret}
        \forall f\in\mathbb{H},~\mathrm{Reg}(f)
        =\sum^T_{t=1}[\ell(f_t({\bf x}_t),y_t)-\ell(f({\bf x}_t),y_t)].
    \end{equation}
    An effective algorithm must ensure $\mathrm{Reg}(f)=o(T)$,
    which implies the average loss of the algorithm
    converges to that of the optimal hypothesis.
    In this paper,
    $\ell(\cdot,\cdot)$ is the Hinge loss function,
    and $y_t\in\{-1,1\}$.

    The minimax lower bound on the regret is $\Omega(\sqrt{T})$
    for the Hinge loss function \cite{Abernethy2008Optimal}.
    The online gradient descent (OGD) algorithm \cite{Zinkevich2003Online}
    enjoys a $O(\sqrt{T})$ upper bound which is optimal w.r.t. $T$.
    OGD still suffers two weaknesses.
    (i) It nearly stores all of the observed examples,
    and thus suffers a $O(dt)$ computational complexity
    (i.e., space and per-round time complexity),
    which is prohibitive for large-scale datasets.
    (ii)
    The $O(\sqrt{T})$ bound is too pessimistic for certain benign environments.
    In practice,
    the data might be learnt easily or have some intrinsic structures
    that can be used to circumvent the $\Omega(\sqrt{T})$ barrier.
    Most of the previous algorithms only address the first weakness.
    Such algorithms store limited examples or construct explicit feature mapping.
    For instance,
    the BOGD and FOGD algorithm \cite{Zhao2012Fast,Lu2016Large}
    suffer a $O(dB)$ computational complexity
    and achieve a $O(\sqrt{T}+\frac{T}{\sqrt{B}})$
    \footnote{
    The regret bound of BOGD is obtained from Eq.(14) in \cite{Zhao2012Fast}.}
    regret bound,
    where $B$ is the size of budget or the number of random features.
    The NOGD algorithm \cite{Lu2016Large}
    which uses Nystr\"{o}m approach to construct explicit feature mapping,
    enjoys a regret of $O(\sqrt{T}+\frac{T}{B})$
    at a computational complexity of $O(dB+B^2)$ .
    The above algorithms reduce the computational complexity at the expense of regret bound.
    The SkeGD algorithm \cite{Zhang2019Incremental}
    which uses randomized sketching to construct explicit feature mapping,
    enjoys a regret of $O(\sqrt{TB})$
    at a computational complexity of $O(dB\mathrm{poly}(\ln{T}))$.
    Although the results are better for a constant $B$,
    they become worse in the case of $B=\Theta(T^\mu)$, $0<\mu < 1$.

    The only algorithm that addresses both weaknesses simultaneously
    is $\mathrm{B}\mathrm{(AO)}_2\mathrm{KS}$ \cite{Liao2021High}.
    It achieves a regret of $O((\mathcal{A}_TT\ln{T})^{\frac{1}{4}})$
    at a computational complexity of $O(d\sqrt{\mathcal{A}_TT\ln{T}})$.
    The kernel alignment, $\mathcal{A}_T$,
    measures how well the kernel function matches with the data.
    If we choose a good kernel function,
    then $\mathcal{A}_T=o(T)$ is possible.
    In this case,
    $\mathrm{B}\mathrm{(AO)}_2\mathrm{KS}$ circumvents the $\Omega(\sqrt{T})$ barrier
    and only suffers a $o(T)$ computational complexity.
    If we choose a bad kernel function,
    $\mathrm{B}\mathrm{(AO)}_2\mathrm{KS}$ still nearly matches the results of OGD.
    A natural question arises:
    \textit{Is it possible to achieve a regret of $O(\sqrt{\mathcal{A}_T})$
    at a computational complexity of $O(\mathrm{poly}(\ln{T}))$?}
    An algorithm with such characteristics
    would constitute a significant improvement on the existing algorithms, including
    OGD, SkeGD and $\mathrm{B}\mathrm{(AO)}_2\mathrm{KS}$.

    In this paper,
    we propose an algorithm, named POMDR,
    and affirmatively answer the question under some mild assumption.
    Our results depend on how fast the eigenvalues of the kernel matrix decay.
    If the eigenvalues decay exponentially,
    then POMDR enjoys a regret of $O(\sqrt{\mathcal{A}_T})$
    at a computational complexity of $O(d\ln{T}+\ln^2{T})$.
    Otherwise, POMDR enjoys a regret of $O(\sqrt{\mathcal{A}_T}+\frac{\sqrt{T\mathcal{A}_T}}{\sqrt{B}})$
    at a computational complexity of $O(dB)$.
    $B$ is a tunable parameter.
    If $B=\sqrt{\mathcal{A}_TT}$,
    then POMDR still improves the results of $\mathrm{B}\mathrm{(AO)}_2\mathrm{KS}$.
    Table \ref{tab:AAAI2023:comparison_results} summarizes the related results.

    \begin{table}[!t]
      \centering
      \begin{tabular}{l|l|r}
      \toprule
        \multirow{2}{*}{Algorithm}       & \multirow{2}{*}{Regret bound}  & Computational    \\
                                         &                                & complexity    \\
      \toprule
        OGD             & $O(\sqrt{T})$ & $O(dT)$\\
        BOGD            & $O(\sqrt{T}+\frac{T}{\sqrt{B}})$ & $O(dB)$\\
        NOGD            & $O(\sqrt{T}+\frac{T}{B})$ & $O(dB+B^2)$\\
        SkeGD           & $O(\sqrt{TB})$ & $O(dB\mathrm{poly}(\ln{T}))$\\
        $\mathrm{B}\mathrm{(AO)}_2\mathrm{KS}$
        & $O((\mathcal{A}_TT\ln{T})^{\frac{1}{4}})$
                        & $O(d\sqrt{\mathcal{A}_TT\ln{T}})$\\
      \hline
        \multirow{2}{*}{POMDR}& ${\color{blue}O(\sqrt{\mathcal{A}_T})}$
                        & ${\color{blue}O((d\wedge\ln{T})\ln{T})}$\\
                        & ${\color{blue}O(\sqrt{\mathcal{A}_T}+\frac{\sqrt{T\mathcal{A}_T}}{\sqrt{B}})}$
                        & ${\color{blue}O(dB)}$\\
      \bottomrule
      \end{tabular}
      \caption{Regret bound and computational complexity
      of online kernel learning algorithms in the regime of the Hinge loss.
      $B\leq T$ is the size of budget.
      $\max\{a,b\}=a\wedge b$.}
      \label{tab:AAAI2023:comparison_results}
    \end{table}

    Our algorithm combines two techniques, namely
    optimistic mirror descent (OMD) \cite{Chiang2012Online,Rakhlin2013Online}
    and the \textit{approximate linear dependence (ALD) condition} \cite{Engel2004The},
    and gives a new budget maintaining approach.
    OMD achieves the $O(\sqrt{\mathcal{A}_T})$ regret bound.
    The ALD condition ensures the $O(d\ln{T}+\ln^2{T})$ computational complexity.
    The main challenge is to analyze the size of the budget maintained by the ALD condition.
    Previous analysis assumed that examples are i.i.d,
    which is commonly violated in an online learning setting.
    We give a new and cleaner analysis for the size of budget.
    To be specific,
    if the eigenvalues of the kernel matrix decay exponentially,
    then the size is $O(\ln{\frac{T}{\alpha}})$,
    where $\alpha> 0$ is a threshold parameter.
    If the eigenvalues decay polynomially with degree $p\geq 1$,
    then the size is $O((T/\alpha)^{\frac{1}{p}})$.

    We extend our algorithm to batch learning,
    and give a $O(\frac{1}{T}\sqrt{\mathbb{E}[\mathcal{A}_T]})$ excess risk bound in expectation.
    Such a result may also circumvent the $\Omega(\frac{1}{\sqrt{T}})$ barrier \cite{Srebro2010Smoothness}.

\section{Related Work}

    \citet{Engel2004The} first
    used the ALD condition to control the budget of the kernel recursive least-squares algorithm,
    and proved the budget is bounded.
    The Projectron algorithm \cite{Orabona2008The} also used the ALD condition to maintain the budget
    and proved a same result.
    An significant improvement on the size of the budget was given by \citet{Sun2012On}.
    They proved similar results with our results.
    However,
    their results are not suitable for online learning,
    since they assume that the examples are i.i.d.
    Our results do not require such an assumption and are more general.

    Our results reveal new trade-offs between regret bound and computational costs
    for online kernel learning in the regime of the Hinge loss function.
    Previous work only focus on loss functions with strong curvature properties,
    such as smoothness and exp-concave.
    For exp-concave loss functions,
    the PROS-N-KONS algorithm \cite{Calandriello2017Efficient}
    which combines the online Newton step algorithm \cite{Hazan2007Logarithmic}
    and KORS \cite{Calandriello2017Second},
    achieves a regret of $O(\mathrm{d}_{\mathrm{eff}}(\gamma)\ln^2{T})$
    at a computational complexity of $O(\mathrm{d}^2_{\mathrm{eff}}(\gamma)\ln^4{T})$.
    $\mathrm{d}_{\mathrm{eff}}(\gamma)$ is the effective dimension of kernel matrix.
    If the eigenvalues of kernel matrix decay exponentially,
    then $\mathrm{d}_{\mathrm{eff}}(\gamma)=O(\ln{T})$.
    For smooth loss functions,
    the OSKL algorithm \cite{Zhang2013Online}
    achieves a regret of $O(\sqrt{L^\ast_T})$ at a computational complexity of $O(dL^\ast_T)$.
    $L^\ast_T$ is the cumulative losses of optimal hypothesis.
    The above two types of regret bound can also circumvent the $\Omega(\sqrt{T})$ barrier,
    but are not suitable for the Hinge loss function
    that does not enjoy strong curvature properties.

    In batch learning setting,
    the examples are i.i.d. sampled from a fixed distribution.
    In this paper,
    our problem setting captures the classical support vector machines (SVM).
    Our algorithm solves the SVM in an online approach
    and outputs an approximate solution $\bar{f}$.
    We prove that $\bar{f}$ achieves
    a $O(\frac{1}{T}\sqrt{\mathbb{E}[\mathcal{A}_T]})$ excess risk bound.
    The Pegasos algorithm \cite{Shalev-Shwartz2007Pegasos}
    can only achieves a $O(\frac{1}{\sqrt{T}})$ excess risk bound.
    If the eigenvalues of the kernel matrix decay exponentially,
    then the time complexity of our algorithm is $O(T\cdot(d\ln{T}+\ln^2{T}))$,
    while the time complexity of Pegasos is $O(dT^2)$.

\section{Problem Setting}

    Let $\mathcal{X}=\{\mathbf{x}\in\mathbb{R}^d\vert\Vert\mathbf{x}\Vert_2<\infty\}$
    and $\mathcal{I}_T=\{(\mathbf{x}_t,y_t)_{t\in[T]}\}$ be a sequence of examples,
    where $[T] = \{1,\ldots,T\}$,
    $\mathbf{x}_t\in\mathcal{X}, y_t\in \{-1,1\}$.
    Let $\kappa(\cdot,\cdot):\mathbb{R}^d \times \mathbb{R}^d \rightarrow [A,D]$
    be a positive semidefinite kernel function,
    where $A>0$.
    Denote by $\mathcal{H}$ the RKHS associated with $\kappa$,
    such that
    (i) $\langle f,\kappa({\bf x},\cdot)\rangle_{\mathcal{H}}=f({\bf x})$;
    (ii) $\mathcal{H}=\overline{\mathrm{span}(\kappa({\bf x}_t,\cdot): t\in[T])}$.
    We define $\langle\cdot,\cdot\rangle_{\mathcal{H}}$ as the inner product in $\mathcal{H}$,
    which induces the norm $\Vert f\Vert_{\mathcal{H}}=\sqrt{\langle f,f\rangle_{\mathcal{H}}}$.
    Denote by $\mathbb{H}=\{f\in\mathcal{H}\vert \Vert f\Vert_{\mathcal{H}}\leq U\}$.
    The hinge loss function is $\ell(f(\mathbf{x}),y)=\max\{0,1-yf(\mathbf{x})\}$.

    Let $\psi(f)$ be a strongly convex regularizer defined on $f\in\mathcal{H}$.
    Denote by $\mathcal{D}_{\psi}(f,g)$ the Bregman divergence,
    $$
        \mathcal{D}_{\psi}(f,g)
        =\psi(f)-\psi(g)-\langle\nabla\psi(g),f-g\rangle.
    $$
    The protocol of online learning in $\mathbb{H}$ is as follows:
    at any round $t$,
    an adversary first sends an instance ${\bf x}_t\in\mathcal{X}$.
    An learner uses an algorithm to choose a hypothesis $f_t\in\mathbb{H}$,
    and makes the prediction $\mathrm{sign}(f_t({\bf x}_t))$.
    Then the adversary reveals the label $y_t$.
    We aim to minimize the regret w.r.t. any $f\in\mathbb{H}$,
    denoted by $\mathrm{Reg}(f)$ which is defined in \eqref{eq:AAAI23:definition_regret}.

    OGD achieves $\mathrm{Reg}(f)=O(U\sqrt{T})$ which is optimal in the worst-case.
    We will prove a data-dependent regret bound,
    that is, $\mathrm{Reg}(f)=O(U\sqrt{A_T})$,
    where $\mathcal{A}_T$ is called \textit{kernel alignment} \cite{Liao2021High} defined as follows
    $$
        \mathcal{A}_T=\sum^T_{t=1}\kappa(\mathbf{x}_t,\mathbf{x}_t)
        -\frac{1}{T}{\bf Y}^\top_T{\bf K}_T{\bf Y}_T.
    $$
    ${\bf K}_T$ is the kernel matrix on $\mathcal{I}_T$
    and ${\bf Y}_T=(y_1,\ldots,y_T)^\top$.
    ${\bf Y}^\top_T{\bf K}_T{\bf Y}_T$ is
    called \textit{kernel polarization} \cite{Baram2005Learning},
    a classical kernel selection criterion.
    If $\mathbf{K}_{T}$ is the ideal kernel matrix $\mathbf{Y}_T\mathbf{Y}^\top_T$,
    then $\mathcal{A}_T=\Theta(1)$.
    More generally,
    if $\kappa$ matches well with the data,
    then we expect that $\mathcal{A}_T\ll T$.
    In this case,
    the $O(U\sqrt{A_T})$ bound circumvents the $\Omega(U\sqrt{T})$ barrier.
    In the worse case, i.e.,
    $\mathcal{A}_T\approx T$,
    we still have $O(U\sqrt{A_T})=O(U\sqrt{T})$.

\section{Algorithm}

    Our algorithm consists of two phases.
    The first one is named Projected Optimistic Mirror Descent (POMD).
    The second one is Optimistic Mirror Descent with Removing (OMDR).

\subsection{POMD}

    POMD is based on the optimistic mirror descent framework
    (OMD) \cite{Chiang2012Online,Rakhlin2013Online}.
    For all $t \geq 1$,
    OMD maintains $f'_{t-1},f_{t}\in\mathcal{H}$.
    Let $S_t$ be a budget storing a subset of $\{(\mathbf{x}_{\tau},y_{\tau})^{t-1}_{\tau=1}\}$
    and $\nabla_t=\nabla\ell(f_t(\mathbf{x}_t),y_t)$.
    OMD is defined as follows,
    \begin{align}
        f_{t}=&\mathop{\arg\min}_{f\in\mathcal{H}}\left\{
        \langle f,\bar{\nabla}_{t}\rangle
        +\mathcal{D}_{\psi_{t}}(f,f'_{t-1})\right\},
        \label{eq:COLT2021:exact_AOMD:first_OMD_hypothesis_updating}\\
        f'_{t}=&\mathop{\arg\min}_{f\in\mathcal{H}}\left\{
        \langle f,\nabla_{t}\rangle+\mathcal{D}_{\psi_{t}}(f,f'_{t-1})\right\},
        \label{eq:COLT2021:exact_AOMD:second_OMD_hypothesis_updating}
    \end{align}
    where $\bar{\nabla}_t$ is an optimistic estimator of $\nabla_t$.
    When receiving $\mathbf{x}_t$,
    we first compute $f_{t}$
    and then predict $\hat{y}_t=\mathrm{sign}(f_{t}(\mathbf{x}_t))$.
    After observing $y_t$,
    we compute the gradient $\nabla_t$
    and execute the update
    \eqref{eq:COLT2021:exact_AOMD:second_OMD_hypothesis_updating}.
    It is obvious that, if $\nabla_t=0$,
    then $f'_t=f'_{t-1}$.
    In this case,
    we do not add $(\mathbf{x}_t,y_t)$ into $S_t$.
    If $\nabla_t=-y_t\kappa(\mathbf{x}_t,\cdot)$,
    then $f'_t\neq f'_{t-1}$
    and we must add $(\mathbf{x}_t,y_t)$ into $S_t$
    which increases the memory cost.
    To address this issue,
    we use the ALD condition \cite{Engel2004The}
    to maintain the budget $S_t$.

    We consider any round $t$
    at which $\nabla_t=-y_t\kappa(\mathbf{x}_t,\cdot)$.
    The ALD condition measures
    whether $\kappa(\mathbf{x}_t,\cdot)$ is approximate linear dependence with the instances in $S_t$.
    Let ${\bm \Phi}_{S_t}=(\kappa(\mathbf{x},\cdot)_{\mathbf{x}\in S_t})$.
    We first compute the projection error
    \begin{equation}
    \label{eq:AAAI23:projection_ALD}
        \left(\min_{{\bm \beta}\in\mathbb{R}^{\vert S_t\vert}}
        \left\Vert {\bm \Phi}_{S_t}{\bm \beta}-\kappa(\mathbf{x}_t,\cdot)\right\Vert^2_{\mathcal{H}}\right)
        =:\alpha_t.
    \end{equation}
    The solution
    \footnote{If $S_t=\emptyset$, then we set ${\bm \beta}^\ast_t=0$ and $\alpha_t=D$.}
    is
    $$
        {\bm \beta}^\ast_t = {\bf K}^{-1}_{S_t}{\bm \Phi}^\top_{S_t}\kappa(\mathbf{x}_t,\cdot),
    $$
    where ${\bf K}_{S_t}$ is the kernel matrix defined on $S_t$.
    If $\alpha_t=0$,
    then $\kappa(\mathbf{x}_t,\cdot)$ is linear dependence with the instances in $S_t$.
    Thus we do not add $(\mathbf{x}_t,y_t)$ into $S_t$.
    Linear dependence is a strong condition.
    We introduce a threshold for $\alpha_t$
    and define the ALD condition as follows
    \begin{equation}
    \label{eq:AAAI23:POMD:second_updating:ALD}
        \mathrm{ALD}_t:\sqrt{\alpha_t} \leq \sqrt{D}\cdot T^{-\zeta},~\zeta\in(0,1],
    \end{equation}
    where $\sup_{{\bf x}\in\mathcal{X}}\kappa(\mathbf{x},\mathbf{x})\leq D$.
    If $\mathrm{ALD}_t$ holds,
    then $\kappa(\mathbf{x}_t,\cdot)$ is approximate linear dependence with the instances in $S_t$.
    We can safely replace $\kappa(\mathbf{x}_t,\cdot)$
    with ${\bm \Phi}_{S_t}{\bm \beta}^\ast_t$,
    and do not add $(\mathbf{x}_t,y_t)$ into $S_t$.
    We replace \eqref{eq:COLT2021:exact_AOMD:second_OMD_hypothesis_updating}
    with \eqref{eq:AAAI23:POMD:second_updating:a},
    \begin{equation}
    \label{eq:AAAI23:POMD:second_updating:a}
        \left\{
        \begin{array}{l}
            f'_{t}=\mathop{\arg\min}_{f\in\mathbb{H}}\left\{
            \langle f,\hat{\nabla}_{t}\rangle+\mathcal{D}_{\psi_{t}}(f,f'_{t-1})\right\},\\
            \hat{\nabla}_{t}=-y_t{\bm \Phi}_{S_t}{\bm \beta}^\ast_t.
        \end{array}
        \right.
    \end{equation}
    If $\mathrm{ALD}_t$ does not hold,
    that is, $\kappa(\mathbf{x}_t,\cdot)$
    can not be approximated by ${\bm \Phi}_{S_t}{\bm \beta}^\ast_t$,
    then we add $(\mathbf{x}_t,y_t)$ into $S_t$
    and replace \eqref{eq:COLT2021:exact_AOMD:second_OMD_hypothesis_updating}
    with \eqref{eq:AAAI23:POMD:second_updating:b},
    \begin{equation}
    \label{eq:AAAI23:POMD:second_updating:b}
        f'_{t}=\mathop{\arg\min}_{f\in\mathbb{H}}\left\{
        \langle f,\nabla_{t}\rangle+\mathcal{D}_{\psi_{t}}(f,f'_{t-1})\right\}.
    \end{equation}
    We set
    $\psi_t(f)=\frac{1}{2\lambda_t}\Vert f\Vert^2_{\mathcal{H}}$.
    Then \eqref{eq:COLT2021:exact_AOMD:first_OMD_hypothesis_updating},
    \eqref{eq:AAAI23:POMD:second_updating:a}
    and \eqref{eq:AAAI23:POMD:second_updating:b} become gradient descent.
    The learning rate $\lambda_t$ is defined as follows
    \begin{align*}
        \lambda_t=&\frac{U}{\sqrt{3+\sum^{t-1}_{\tau=1}\max\{\Vert\tilde{\nabla}_{\tau}-\bar{\nabla}_{\tau}\Vert^2_{\mathcal{H}}
        -\Vert\bar{\nabla}_{\tau}\Vert^2_{\mathcal{H}},0\}}},\\
        \tilde{\nabla}_{\tau}=&\hat{\nabla}_{\tau},~\mathrm{if}~\mathrm{ALD}_\tau~\mathrm{holds},\\
        \tilde{\nabla}_{\tau}=&\nabla_{\tau},~\mathrm{otherwise}.
    \end{align*}

    Note that POMD allows $f_t\in\mathcal{H}$ and $f'_t\in \mathbb{H}$,
    while OMD requires both $f_t$ and $f'_t$ belong to $\mathcal{H}$ (or $\mathbb{H}$).
    The slight modification gives improved regret bound.
    More precisely,
    there will be a negative term, $-\Vert\bar{\nabla}_t\Vert^2_{2}$,
    in the regret bound.
    We use projection to ensure $f'_t\in \mathbb{H}$,
    and thus name this procedure POMD (Projected OMD).

\subsection{OMDR}

    At any round $t$,
    we can compute ${\bf K}^{-1}_{S_t}$ incrementally.
    Thus the space and per-round time complexity of POMD are $O(d\vert S_t\vert+\vert S_t\vert^2)$.
    We will prove that $\vert S_T\vert$ depends on
    how fast the eigenvalues of the kernel matrix decay.
    If the eigenvalues decay exponentially,
    then $\vert S_T\vert=O(\ln{T})$.
    More rigorous results are given in Theorem \ref{thm:AAAI2023:size_budget}.
    Otherwise,
    $\vert S_T\vert$ may be large
    and POMD suffers high computational costs.
    To address this issue,
    we set a threshold on $\vert S_t\vert$.

    We first execute POMD.
    Let $B_0=\Theta(\ln{T})$ and
    $\bar{t}=\min_{t\in [T]}\{t:\vert S_t\vert=B_0\}$.
    If $\bar{t}<T$,
    then the eigenvalues of the kernel matrix may not decay exponentially.
    In this case,
    we stop executing POMD for $t\geq \bar{t}+1$.
    To be specific,
    we always execute the following two steps,
    \begin{align}
        f_{t}=&\mathop{\arg\min}_{f\in\mathcal{H}}\left\{
        \langle f,\bar{\nabla}_{t}\rangle
        +\mathcal{D}_{\psi_{t}}(f,f'_{t-1})\right\},\nonumber\\
        f'_{t}=&\mathop{\arg\min}_{f\in\mathbb{H}}\left\{
        \langle f,\nabla_{t}\rangle+\mathcal{D}_{\psi_{t}}(f,f'_{t-1})\right\}
        \label{eq:AAAI23:POMD:second_updating:c}.
    \end{align}
    If $\nabla_t\neq 0$,
    we add $(\mathbf{x}_t,y_t)$ into $S_t$.
    Let $B>B_0$ be another threshold.
    If $\vert S_t\vert =B$,
    then we remove ${B}/{2}$ examples from $S_t$.
    Denote by $\{t_i\in[\bar{t}+1,T]\}^N_{i=1}$ the time instances
    at which $\vert S_{t_i}\vert=B$.
    Let $S_{t_i}=\{(\mathbf{x}_{\tau_j},y_{\tau_j})^B_{j=1}\}$, where $\tau_j<t_i$.
    We will delete $\{(\mathbf{x}_{\tau_j},y_{\tau_j})^B_{j=B/2+1}\}$ from $S_{t_i}$.
    According to the representer theorem,
    we can rewrite
    $f'_{t_i}=\sum^{B}_{j=1}a_{\tau_j}\kappa(\mathbf{x}_{\tau_j},\cdot)$.
    For each $j\geq {B}/{2}+1,\ldots,B$,
    we define $\mathbf{x}_{\tau_{\underline{j}}}$ as follows
    $$
        \mathbf{x}_{\tau_{\underline{j}}}
        =\mathop{\arg\max}_{i=1,\ldots,B/2}\kappa(\mathbf{x}_{\tau_i},\mathbf{x}_{\tau_{j}}).
    $$
    To keep more information as possible,
    we construct $\underline{f}'_{t_i}$ by
    \begin{equation}
    \label{eq:AAAI2023:OMDR:projection}
        \underline{f}'_{t_i}=\sum^{B/2}_{j=1}a_{\tau_j}\kappa(\mathbf{x}_{\tau_j},\cdot)
        +\sum^{B}_{j=B/2+1}a_{\tau_j}\kappa(\mathbf{x}_{\tau_{\underline{j}}},\cdot).
    \end{equation}
    Now we redefine $f'_{t_i}=\underline{f}'_{t_i}\cdot \frac{U}{\Vert \underline{f}'_{t_i}\Vert_{\mathcal{H}}}$.
    Next we reset the learning rate after the removing operation
    \begin{align}
    \label{eq:AAAI2023:learning_rate}
    \begin{split}
        \lambda_t=&\frac{U}{\sqrt{\max_{t\in[t_j+1,t_{j+1}]}\delta_{t}+\sum^{t-1}_{\tau=t_j+1}\delta_{\tau}}},\\
        \delta_{\tau}=&\max\{\Vert\nabla_{\tau}-\bar{\nabla}_{\tau}\Vert^2_{\mathcal{H}}
        -\Vert\bar{\nabla}_{\tau}\Vert^2_{\mathcal{H}},0\}.
    \end{split}
    \end{align}
    We name this procedure OMDR.
    The space and average per-round time complexity of OMDR is $O(dB)$.

\subsection{Optimistic Estimator $\bar{\nabla}_t$}

    The key of POMD and OMDR is $\bar{\nabla}_{t}$
    which depends on the desired regret bound.
    We will prove that the regret bound depends on
    the following term,
    \begin{align*}
        &\sqrt{\sum^T_{t=1}\delta_t}=\sqrt{\sum^T_{t=1}
            ([\Vert\nabla_t-\bar{\nabla}_t\Vert^2_{\mathcal{H}}-\Vert\bar{\nabla}_t\Vert^2_{\mathcal{H}}]
            \wedge0)}.
    \end{align*}
    To achieve the kernel alignment regret bound,
    the optimal value of $\bar{\nabla}_t$ is
    $\sum^{t-1}_{\tau=1}\frac{-y_\tau}{t-1}\kappa(\mathbf{x}_\tau,\cdot)$.
    However, such a value needs to store the past $t-1$ examples.
    To avoid this issue,
    we define $\bar{\nabla}_t=-\sum^{M_t}_{r=1}\frac{y_{t-r}}{M_t}\kappa(\mathbf{x}_{t-r},\cdot)$
    and $\bar{\nabla}_1=0$.
    Let $M\geq 1$ be a small constant.
    Then we define $M_t=t-1$ for $2\leq t\leq M$
    and $M_t=M$ for $t>M$.
    Our $\bar{\nabla}_t$ only needs to store the past $M$ examples,
    and does not increases the computational costs.
    We name the algorithm POMDR and give the pseudo-code in
    Algorithm \ref{alg:AAAI2023:POMDR}.

    \begin{algorithm}[!t]
        \caption{POMDR}
        \footnotesize
        \label{alg:AAAI2023:POMDR}
        \begin{algorithmic}[1]
        \Require {$U$, $\zeta$, $B$, $B_0$, $M$.}
        \Ensure  {$f'_{0} = \bar{\nabla}_1=0$, $\bar{t}=+\infty$, $S_1=\emptyset$.}
        \For{$t=1,\ldots,T$}
            \State Receive $\mathbf{x}_t$
            \State Compute $\lambda_t$
            \State Compute $f_t$ following \eqref{eq:COLT2021:exact_AOMD:first_OMD_hypothesis_updating}
            \State Make prediction $\hat{y}_t=\mathrm{sign}(f_t({\bf x}_t))$
            \State Receive $y_t$ and compute loss $\ell(f_t({\bf x}_t),y_t)$
            \If{$\ell(f_t({\bf x}_t),y_t)>0$}
                \If{$\bar{t}\geq T$}
                \State Compute ${\bm \beta}^\ast_t = {\bf K}^{-1}_{S_t}{\bm \Phi}^\top_{S_t}\kappa(\mathbf{x}_t,\cdot)$
                \State Compute $\alpha_t$ following \eqref{eq:AAAI23:projection_ALD}
                \If{$\mathrm{ALD}_t$ holds}
                    \State Compute $\hat{\nabla}_t$
                    \State Update $f'_t$ following \eqref{eq:AAAI23:POMD:second_updating:a}
                \Else
                    \State Update $f'_t$ following \eqref{eq:AAAI23:POMD:second_updating:b}
                    \State $S_{t+1}=S_t\cup\{(\mathbf{x}_t,y_t)\}$
                    \State \textbf{if} $\vert S_{t+1}\vert==B_0$,
                            \textbf{then} $\bar{t}=t+1$
                \EndIf
                \Else
                    \State Update $f'_t$ following \eqref{eq:AAAI23:POMD:second_updating:c}
                    \State $S_{t+1}=S_t\cup\{(\mathbf{x}_t,y_t)\}$
                    \State \textbf{if} $\vert S_{t+1}\vert==B$,
                            \textbf{then} compute $\underline{f}'_{t}$ (see \eqref{eq:AAAI2023:OMDR:projection})
                            and $f'_{t}$
                \EndIf
            \EndIf
        \EndFor
    \end{algorithmic}
    \end{algorithm}

\section{Main Results}

    In this section,
    we give the size of the budget generated by the ALD condition
    and the kernel alignment regret bound.

\subsection{The size of Budget}

    \begin{theorem}
    \label{thm:AAAI2023:size_budget}
        Let $S_1=\emptyset$
        and $\mathrm{ALD}_t$ be defined in \eqref{eq:AAAI23:POMD:second_updating:ALD}
        where $\alpha=D\cdot T^{-2\zeta}$ for a certain $\zeta\in(0,1]$.
        For all $t\leq T-1$,
        if $\mathrm{ALD}_t$ does not hold,
        then $S_{t+1}=S_t\cup\{(\mathbf{x}_t,y_t)\}$.
        Otherwise, $S_{t+1}=S_t$.
        Let $\{\lambda_i\}^T_{i=1}$ be the eigenvalues of ${\bf K}_T$ sorted in decreasing order.
        If $\{\lambda_i\}^T_{i=1}$ decay exponentially,
        that is, there is a constant $R_0>0$ and $0<r<1$
        such that $\lambda_i\leq R_0r^{i}$,
        then $\vert S_T\vert\leq 2\frac{\ln{(\frac{C_1R_0}{\alpha})}}{\ln{r^{-1}}}$.
        If $\{\lambda_i\}^T_{i=1}$ decay polynomially,
        that is, there is a constant $R_0>0$ and $p\geq 1$,
        such that $\lambda_i\leq R_0i^{-p}$,
        then $\vert S_T\vert\leq \mathrm{e}\cdot(\frac{C_2R_0}{\alpha})^{\frac{1}{p}}$.
        In both cases,
        $C_1$ and $C_2$ are constants,
        and $R_0=\Theta(T)$.
    \end{theorem}

    It is worth mentioning that
    we use the ALD condition to update $S_t$ only if $\nabla_t\neq 0$.
    In this case, Theorem \ref{thm:AAAI2023:size_budget} still holds,
    since our proof is independent of the condition $\nabla_t\neq 0$.
    Initial analyses of the ALD condition only proved that
    the budget is bounded \cite{Engel2004The}.
    An improved result was given by \citet{Sun2012On}.
    They proved similar results with Theorem \ref{thm:AAAI2023:size_budget}.
    There are three main differences
    between their results \cite{Sun2012On} and our results.
    (i) Their results require that $\{\mathbf{x}_t\}^T_{t=1}$
    are sampled i.i.d. from a fixed distribution,
    which may not hold in an online learning setting.
    (ii) Their results hold in a high-probability, while our results are deterministic.
    (iii) Our analyses are simpler.

    The KORS algorithm \cite{Calandriello2017Second}
    uses the ridge leverage scores to construct sampling probability,
    and then randomly adds examples.
    KORS ensures $\vert S_T\vert=O(\mathrm{d}_{\mathrm{eff}}(\gamma)\ln^2{T})$,
    where $\mathrm{d}_{\mathrm{eff}}(\gamma)
    =\mathrm{tr}({\bf K}_T({\bf K}_T+\gamma{\bf I})^{-1})$
    is the effective dimension.
    If $\lambda_i\leq R_0r^{i}$,
    then $\mathrm{d}_{\mathrm{eff}}(\gamma)=O(\ln(R_0/\gamma))$.
    If $\lambda_i\leq R_0i^{-p}$,
    then $\mathrm{d}_{\mathrm{eff}}(\gamma)=O((R_0/\gamma)^{\frac{1}{p}})$.
    KORS is worse than the ALD condition by a factor of order $O(\ln^2{T})$.

    \begin{proof}[Proof Sketch of Theorem \ref{thm:AAAI2023:size_budget}]
        We first state a key lemma
        which proves that the eigenvalues of a Hermitian matrix
        ${\bf A}\in\mathbb{C}^{T\times T}$
        are interlaced with those of any principal submatrix
        ${\bf B}\in\mathbb{C}^{m\times m}$, $m\leq T-1$.

        \begin{lemma}[Theorem 4.3.28 in \citet{HornMatrix2012Matrix}
        \footnote{The reviewer brings \citet{HornMatrix2012Matrix} to our attention.
        Theorem 4.3.28 in \cite{HornMatrix2012Matrix} makes our proof cleaner.
        Our initial manuscripts uses Theorem 1 in \citet{Hwang2004Cauchy}.}]
        \label{lemma:AAAI2023:Hwang2004Cauchy}
            Let ${\bf A}$ be a Hermitian matrix of order $T$,
            and let ${\bf B}$ be a principle submatrix of ${\bf A}$ of order $m$.
            If $\lambda_T\leq \lambda_{T-1}\leq\ldots\leq\lambda_2\leq\lambda_1$
            lists the eigenvalues of ${\bf A}$
            and $\mu_m\leq \mu_{m-1}\leq\ldots\leq\mu_{1}$
            lists the eigenvalues of ${\bf B}$, then
            $$
                \lambda_{i+T-m}\leq \mu_i\leq\lambda_i,~i=1,2,\ldots,m.
            $$
        \end{lemma}
        The kernel matrix ${\bf K}_T$ which is positive semidefinite and real symmetric,
        satisfies Lemma \ref{lemma:AAAI2023:Hwang2004Cauchy}.
        Observing that
        \begin{align*}
            \alpha_t
            =&\kappa(\mathbf{x}_t,\mathbf{x}_t)-
            ({\bm \Phi}^\top_{S_t}\kappa(\mathbf{x}_t,\cdot))^\top
            {\bf K}^{-1}_{S_t}{\bm \Phi}^\top_{S_t}\kappa(\mathbf{x}_t,\cdot)\\
            =&\frac{\mathrm{det}({\bf K}_{S_t\cup\{(\mathbf{x}_t,y_t)\}})}{\mathrm{det}({\bf K}_{S_t})}.
        \end{align*}
        Let $\vert S_T\vert=k$ and $\{t_j\}^{k}_{j=1}$ be the set of time index at which
        $\mathrm{ALD}_{t_j}$ does not hold.
        ${\bf K}_{S_{t_j}}$ is a $j$-order principle submatrix of ${\bf K}_T$.
        Let $\{\lambda^{(j)}_i\}^{j}_{i=1}$ be the eigenvalues of ${\bf K}_{S_{t_j}}$,
        where $j=1,\ldots,k$.
        Let $\{\lambda_j\}^T_{j=1}$ be the eigenvalues of ${\bf K}_{T}$.

        If $k\leq 2$,
        Theorem \ref{thm:AAAI2023:size_budget} always holds.
        Next we assume that $k> 3$.
        We focus on the eigenvalues of ${\bf K}_{S_{t_k}}$.
        We first consider the case, $\lambda_j\leq R_0r^{j}$, $j=1,\ldots,T$.
        Lemma \ref{lemma:AAAI2023:Hwang2004Cauchy}
        gives $\lambda^{(k)}_k\leq \lambda_k\leq R_0r^{k}$.
        \begin{itemize}
          \item \textbf{Case 1}: $\lambda^{(k)}_k= R_0r^{k}$.\\
                If $\mathrm{ALD}_t$ does not hold, then it must be
                $\alpha_t > DT^{-2\zeta}=:\alpha$.
                We have
                \begin{equation}
                \label{eq:AAAI23:proof_ALD:case_2_condtion}
                    \frac{\mathrm{det}(\mathbf{K}_{S_{t_{k}}})}{\mathrm{det}(\mathbf{K}_{S_{t_{1}}})}
                    =\frac{\lambda^{(k)}_{k}
                    \cdot\prod^{k-1}_{j=1}\lambda^{(k)}_{j}}
                    {\lambda^{(1)}_{1}}
                    >\alpha^{k-1}.
                \end{equation}
                Let $T=k$ and $m=1$ in
                Lemma \ref{lemma:AAAI2023:Hwang2004Cauchy}.
                We can obtain $\lambda^{(1)}_1\geq \lambda^{(k)}_k$.
                Thus it must be
                \begin{equation}
                \label{eq:AAAI23:proof_ALD:case_1_condtion}
                    \alpha^{k-1}<\prod^{k-1}_{j=1}\lambda^{(k)}_j.
                \end{equation}
                Lemma \ref{lemma:AAAI2023:Hwang2004Cauchy}
                gives $\lambda^{(k)}_j\leq \lambda_j, j=1,\ldots,k-1$.\\
                Since $\lambda_j\leq R_0r^{j}$,
                we have
                $$
                    R^{k-1}_0r^{k(k-1)}< \prod^{k-1}_{j=1}\lambda^{(k)}_j
                    \leq \prod^{k-1}_{j=1}\lambda_j
                    \leq R^{k-1}_0r^{\frac{k(k-1)}{2}}.
                $$
                Solving \eqref{eq:AAAI23:proof_ALD:case_1_condtion} gives
                $
                    k<2\frac{\ln{(\frac{R_0}{\alpha})}}{\ln{r^{-1}}}.
                $
          \item \textbf{Case 2}: $\lambda^{(k)}_k< R_0r^{k}$.\\
                We will prove that if $k$ is large, then there is a contradiction.
                We start with \eqref{eq:AAAI23:proof_ALD:case_2_condtion}.
                Rearranging terms gives
                \begin{align*}
                    \lambda^{(k)}_{k}
                    >\frac{\lambda^{(1)}_{1}\alpha^{k-1}}{\prod^{k-1}_{j=1}\lambda^{(k)}_{j}}
                    \geq&A\alpha^{k-1}
                    \cdot R^{-k+1}_0\cdot r^{-\frac{k(k-1)}{2}}\\
                    =&R_0r^{k}\cdot A\alpha^{k-1}\cdot R^{-k}_0r^{-k-\frac{k(k-1)}{2}},
                \end{align*}
                where the second inequality comes from
                the fact $\lambda^{(k)}_j\leq \lambda_j, j=1,\ldots,k-1$,
                and $\lambda^{(1)}_{1}=\kappa({\bf x}_{t_1},{\bf x}_{t_1})\geq A$.
                Let
                $$
                    \alpha^{k-1}\cdot R^{-k}_0r^{-k-\frac{k(k-1)}{2}}>\frac{1}{A}.
                $$
                Solving for $k$ yields
                $k>2\frac{\ln(\frac{C_1R_0}{\alpha})}{\ln{r^{-1}}}-2$,
                where $C_1$ is a constant depending on $A$.
                In this case,
                we further obtain $\lambda^{(k)}_{k}>R_0r^{k}$
                which contradicts with the condition $\lambda^{(k)}_k<R_0r^{k}$.
                Thus it must be
                $$
                    k\leq 2\frac{\ln(\frac{C_1R_0}{\alpha})}{\ln{r^{-1}}}-2.
                $$
        \end{itemize}
        Combining the two cases, we conclude the first statement.\\
        Next we consider the case,
        $\lambda_j\leq R_0j^{-p}$, $j=1,\ldots,T$.
        \begin{itemize}
          \item \textbf{Case 1}: $\lambda^{(k)}_k= R_0k^{-p}$.\\
                We start with \eqref{eq:AAAI23:proof_ALD:case_1_condtion}.
                First we can obtain
                $$
                    R^{k-1}_0k^{-p(k-1)}< \prod^{k-1}_{j=1}\lambda^{(k)}_j
                    \leq R^{k-1}_0\prod^{k-1}_{j=1}j^{-p}.
                $$
                According to \eqref{eq:AAAI23:proof_ALD:case_1_condtion},
                we obtain a necessary condition
                $$
                    ((k-1)!)^p < R^{k-1}_0\cdot(\alpha^{-1})^{k-1}.
                $$
                Using Stirling's formula, i.e., $k!\approx \sqrt{2\pi k}(k/\mathrm{e})^k$,
                we can obtain
                $
                    k< \mathrm{e}\cdot \left(\frac{R_0}{\alpha}\right)^{\frac{1}{p}}+1.
                $
          \item \textbf{Case 2}: $\lambda^{(k)}_k< R_0k^{-p}$.\\
                We start with \eqref{eq:AAAI23:proof_ALD:case_2_condtion}.
                Similarly, we have
                \begin{align*}
                    \lambda^{(k)}_{k}>\frac{\lambda^{(1)}_{1}\alpha^{k-1}}{\prod^{k-1}_{j=1}\lambda^{(k)}_{j}}
                    \geq&R_0 k^{-p}\cdot A\alpha^{k-1}\cdot R^{-k}_0\cdot (k!)^p.
                \end{align*}
                Let
                $
                    \alpha^{k-1}\cdot R^{-k}_0\cdot (k!)^p>\frac{1}{A}.
                $
                Solving the inequality gives
                $$
                    k>\mathrm{e}\cdot A^{-\frac{1}{pk}}R^{\frac{1}{p}}_0(\alpha^{-1})^{\frac{k-1}{kp}}.
                $$
                In this case,
                we have $\lambda^{(k)}_{k}>R_0k^{-p}$
                which contradicts with the condition $\lambda^{(k)}_{k}<R_0k^{-p}$.
                Thus
                $$
                    k\leq\mathrm{e}\cdot A^{-\frac{1}{pk}}R^{\frac{1}{p}}_0(\alpha^{-1})^{\frac{k-1}{kp}}
                    \leq \mathrm{e}\cdot \left(\frac{C_2R_0}{\alpha}\right)^{\frac{1}{p}},
                $$
                where $C_2$ is a constant depending on $A$ and $D$.
        \end{itemize}
        Up to now, we conclude the proof.
    \end{proof}

\subsection{Regret bound}

    \begin{theorem}[Regret bound]
    \label{thm:AAAI2023:loss_regret_bound}
        Let $U>0$ and $B_0>\frac{2}{\ln{r^{-1}}}\ln{(\frac{C_1R_0}{\alpha})}$
        where $\alpha=D\cdot T^{-2\zeta}$.
        Let $\zeta=1$.
        If the eigenvalues of ${\bf K}_T$ decay exponentially,
        then the regret of POMDR satisfies
        \begin{align*}
           \forall f\in\mathbb{H},\quad \mathrm{Reg}(f) \leq 3U\sqrt{\sum^T_{\tau=1}\delta_{\tau}}+9U,
        \end{align*}
        and the space and per-round time complexity is $O(d\ln(T)+\ln^2(T))$.
        Otherwise, the regret of POMDR satisfies
        \begin{align*}
           \mathrm{Reg}(f)
           \leq 3U\sqrt{\sum^T_{\tau=1}\delta_{\tau}}+9U
           +3U\frac{\sqrt{2T\sum^T_{\tau=1}\delta_{\tau}}}{\sqrt{B}},
        \end{align*}
        and the space and per-round time complexity is $O(dB)$.
        $\delta_{\tau}$ is given by \eqref{eq:AAAI2023:learning_rate}.
    \end{theorem}

    Note that POMDR needs to tune $B_0$
    which must satisfy $B_0>\frac{2}{\ln{r^{-1}}}\ln{(\frac{C_1R_0}{\alpha})}$.
    Since $R_0=\Theta(T)$ and $\alpha=DT^{-2}$,
    we can empirically set $B_0=\lceil c\cdot \ln{T}\rceil$ where $c>4$.

    Our regret bound improves the previous optimistic regret bounds.
    Both in constrained and unconstrained case,
    the initial regret bound of OMD is
    $O(\sqrt{
        \sum^{T}_{\tau=1}\Vert {\nabla}_{\tau}-\bar{\nabla}_{\tau}\Vert^2_{2}})$
        \cite{Chiang2012Online},
        which is worse than our bound.
    In the unconstrained case,
    \citet{Cutkosky2019Combining} proposed an algorithm with a regret bound of order
    $$
        O\left(B_T({\bf w})
        \sqrt{\left\{1\wedge\sum^{T}_{\tau=1}[\Vert {\nabla}_{\tau}-\bar{\nabla}_{\tau}\Vert^2_{2}-
        \Vert \bar{\nabla}_{\tau}\Vert^2_{2}]\right\}}\right),
    $$
    where $B_T({\bf w})=\Vert{\bf w}\Vert\sqrt{\ln(T\Vert{\bf w}\Vert)}$.
    However,
    their algorithm can not give similar regret bound in the constrained case.
    Our algorithm can be used to both constrained and unconstrained case.
    In the constrained case,
    \citet{Bhaskara2020Online} proposed an algorithm
    with a regret bound of order
    $$
        O\left(U\ln{(T)}\sqrt{1+\sum^{T}_{\tau=1}([\Vert {\nabla}_{\tau}-\bar{\nabla}_{\tau}\Vert^2_{2}-
        \Vert \bar{\nabla}_{\tau}\Vert^2_{2}]\wedge 0)}\right),
    $$
    which is worse than our bound by a factor of order $O(\ln{T})$.

    \begin{corollary}
    \label{coro:AAAI2023:loss_regret_bound}
        Let
        $\bar{\nabla}_t=-\sum^{M_t}_{r=1}\frac{y_{t-r}}{M_t}\kappa(\mathbf{x}_{t-r},\cdot)$.
        Under the condition of Theorem \ref{thm:AAAI2023:loss_regret_bound},
        if the eigenvalues of ${\bf K}_T$ decay exponentially,
        then the regret of POMDR satisfies
        \begin{align*}
           \forall f\in\mathbb{H},\quad \mathrm{Reg}(f) \leq 6U\sqrt{\mathcal{A}_T+2D}+9U.
        \end{align*}
        Otherwise, the regret of POMDR satisfies,
        $\forall f\in\mathbb{H}$,
        \begin{align*}
           \mathrm{Reg}(f)\leq 6U\sqrt{\mathcal{A}_T+2D}+9U+
           6U\frac{\sqrt{2T(\mathcal{A}_T+2D)}}{\sqrt{B}}.
        \end{align*}
    \end{corollary}
    $\mathrm{B}\mathrm{(AO)}_2\mathrm{KS}$ \cite{Liao2021High}
            achieves a $O((\mathcal{A}_TT\ln{T})^{\frac{1}{4}})$ regret bound with high probability
            and suffers a $O(d\sqrt{\mathcal{A}_TT\ln{T}})$ computational complexity.
    We compare our results with that of $\mathrm{B}\mathrm{(AO)}_2\mathrm{KS}$.
    \begin{itemize}
      \item \textbf{Case 1}: $\{\lambda_i\}^T_{i=1}$ decay exponentially.\\
            POMDR achieves a regret of $O(\sqrt{\mathcal{A}_T})$
            at a computational complexity of $O(d\ln{T}+\ln^2{T})$.
            Thus POMDR significantly improves
            the results of $\mathrm{B}\mathrm{(AO)}_2\mathrm{KS}$.
      \item \textbf{Case 2}: $\{\lambda_i\}^T_{i=1}$ decay polynomially.\\
            Let $B=\sqrt{T\mathcal{A}_T}$.
            Then POMDR achieves a regret of $O((\mathcal{A}_TT)^{\frac{1}{4}})$
            at a computational complexity of $O(d\sqrt{\mathcal{A}_TT})$.
            POMDR improves the regret bound of $\mathrm{B}\mathrm{(AO)}_2\mathrm{KS}$
            by a $O(\ln^{\frac{1}{4}}{T})$ factor,
            and reduces the computational complexity by a $O(\sqrt{\ln{T}})$ factor.
    \end{itemize}
    Next we compare our regret bounds with the worst-case regret bounds.
    OGD \cite{Zinkevich2003Online} achieves a regret of $O(\sqrt{T})$
    at a computational complexity of $O(dT)$.
    Our results are never worse than that of OGD,
    and can beat the results of OGD in the case of $\mathcal{A}_T\ll T$.
    SkeGD \cite{Zhang2019Incremental}
    which uses randomized sketching to construct explicit feature mapping,
    achieves a regret of $O(\sqrt{TB})$
    at a computational complexity of $O(dB\mathrm{poly}(\ln{T}))$.
    The results become worse in the case of $B=\Theta(T^\mu)$, $0<\mu < 1$.

    \begin{remark}
        If we always execute POMD,
        then $\mathrm{Reg}(f)=O(U\sqrt{\mathcal{A}_T})$.
        If $\{\lambda_i\}^T_{i=1}$ decay exponentially,
        then the computational complexity is $O(d\ln{T}+\ln^2{T})$.
        If $\{\lambda_i\}^T_{i=1}$ decay polynomially,
        then the computational complexity is $O(\min\{d(\frac{T}{\alpha})^{\frac{1}{p}}+(\frac{T}{\alpha})^{\frac{2}{p}},T^2\})$.
        The larger $p$ is,
        the smaller the computational complexity will be.
        Since $\alpha=D/{T^2}$,
        the computational complexity is larger than that of OGD for $p\leq 3$.
        This is the reason that we execute OMDR.
    \end{remark}

\section{Extension: Batch Learning}

    Let $\mathcal{D}$ be an unknown distribution over $\mathcal{X}\times\{-1,1\}$.
    In batch learning setting,
    $\{(\mathbf{x}_t,y_t)\}^T_{t=1}$ are i.i.d. sampled from $\mathcal{D}$.
    For any $f\in\mathbb{H}$,
    let $\mathcal{E}(f)=\mathbb{E}_{(\mathbf{x},y)\in\mathcal{D}}
    [\ell(f(\mathbf{x}),y)]$ be the risk (or expected loss) of $f$.
    The goal of batch learning is to learn a hypothesis $\bar{f}\in\mathbb{H}$
    with small risk.
    Similar with the notation of regret,
    we focus on the excess risk defined as follows
    $$
        \mathcal{E}(\bar{f})-\min_{f\in\mathbb{H}}\mathcal{E}(f).
    $$
    Classical statistical learning theory
    analyzes the excess risk of (regularized) empirical risk minimizer,
    i.e.,
    $$
        \hat{f}=\mathop{\arg\min}_{f\in\mathbb{H}}
        \left[\frac{1}{T}\sum^T_{t=1}\ell(f(\mathbf{x}_t),y_t)+\gamma R(f)\right],~
        \gamma\geq 0,
    $$
    where $R(f)$ is a regularization function.
    Our algorithm computes an approximate minimizer $\bar{f}$ as a proxy of $\hat{f}$.
    We will prove the excess risk of $\bar{f}$.
    Our algorithm only passes the data once and thus is computationally efficiently.

    We run POMDR on $\{(\mathbf{x}_t,y_t)^T_{t=1}\}$
    and output a fixed hypothesis $\bar{f}$ for predicting new instances.
    To be specific,
    let $\{f_t\}^T_{t=1}$ be the hypotheses produced by POMDR.
    Then we sample $r\in\{1,2,\ldots,T\}$ uniformly,
    and denote $\bar{f}=f_{r}$.
    The selection of $\bar{f}$
    follows the standard online-to-batch conversion technique \cite{Helmbold1995On}.

    Let $\phi_{\kappa}(\cdot):\mathbb{R}^d\rightarrow\mathcal{H}$
    be the feature mapping induced by $\kappa$,
    and $\mathcal{C}=\mathbb{E}_{{\bf x}\sim \mathcal{D}_{\mathcal{X}}}[\phi_{\kappa}\otimes\phi_{\kappa}]$
    be the covariance operator,
    where $\mathcal{D}_{\mathcal{X}}$ be the marginal distribution on $\mathcal{X}$.
    \begin{theorem}[Excess risk bound]
    \label{thm:AAAI2023:excess_risk_bound}
        Let $\mathcal{I}_T=\{(\mathbf{x}_t,y_t)^T_{t=1}\}$
        be i.i.d. sampled from $\mathcal{D}$
        and $\ell$ be the hinge loss function.
        If the eigenvalues of $\mathcal{C}$ decay exponentially,
        then w.p. at least $1-\delta$,
        the excess risk of $\bar{f}$ satisfies, $\forall f\in\mathbb{H}$,
        \begin{align*}
           \mathbb{E}_{\mathcal{I}_T,r}[\mathcal{E}(\bar{f})]-\mathcal{E}(f)
           \leq 6U\frac{\sqrt{\mathbb{E}_{\mathcal{I}_T}[\mathcal{A}_T]+2D}}{T}+\frac{9U}{T},
        \end{align*}
        and the space and per-round time complexity is $O(d(\ln(T)/\ln{(r\delta^{-1})})
        +(\ln(T)/\ln{(r\delta^{-1})})^2)$.
        Otherwise, w.p. at least $1-\delta$,
        the excess risk satisfies
        \begin{align*}
           &\forall f\in\mathbb{H},\quad
           \mathbb{E}_{\mathcal{I}_T,r}[\mathcal{E}(\bar{f})]-\mathcal{E}(f)\\
           &\leq \frac{6U\sqrt{\mathbb{E}_{\mathcal{I}_T}[\mathcal{A}_T]+2D}}{T}+\frac{9U}{T}+
           6U\frac{\sqrt{2\mathbb{E}_{\mathcal{I}_T}[\mathcal{A}_T]+4D}}{\sqrt{TB}},
        \end{align*}
        and the space and per-round time complexity is $O(dB)$.
    \end{theorem}

    For $L$-Lipschitz loss functions,
    \citet{Srebro2010Smoothness} proved a $\Omega(\frac{LU}{\sqrt{T}})$
    lower bound on the excess risk.
    The Pegasos algorithm \cite{Shalev-Shwartz2007Pegasos}
    achieves the optimal upper bound $O(\frac{LU}{\sqrt{T}})$.
    Let $B=T$.
    Our algorithm gives a data-dependent excess risk bound
    that can circumvent the $\Omega(1/\sqrt{T})$ bound.
    For instance,
    if $\mathbb{E}_{\mathcal{I}_T}[\mathcal{A}_T]=\Theta(T^{\mu})$, $0\leq \mu< 1$,
    then our algorithm achieves a $O(\frac{1}{T^{1-\mu/2}})$ excess risk bound.
    In the worst case, i.e., $\mu=1$,
    our result is still optimal.

\section{Experiments}

    In this section,
    we verify the following two goals.
    \begin{enumerate}[\textbf{G} 1]
      \item If the kernel function is well tuned,
            then $\mathcal{A}_T\ll T$,
            or the eigenvalues decay exponentially in some real datasets.
            In this case,
            the size of budget is $O(\ln{T})$ and our algorithm only executes POMD.
      \item Within a fixed budget size $B$,
            our algorithm shows better or similar prediction performance and running time
            w.r.t. the state-of-the-art algorithms.
    \end{enumerate}

\subsection{Experimental Setups}

    We adopt the Gaussian kernel $\kappa(\mathbf{x},\mathbf{v})=\exp(-\frac{\Vert\mathbf{x}-\mathbf{v}\Vert^2_2}{2\sigma^2})$.
    We choose three classification datasets
    (\textit{w8a:49,749, magic04:19,020, mushrooms:8,124})
    from UCI machine learning repository
    \footnote{\url{http://archive.ics.uci.edu/ml/datasets.php}}.
    More experimental results are shown in the supplementary materials.
    We do not compare with OGD,
    since it suffers $O(dt)$ per-round time complexity which is prohibitive.
    We compare with some variants of OGD,
    including FOGD, NOGD \cite{Lu2016Large} and SkeGD \cite{Zhang2019Incremental}.
    We also compare with $\mathrm{B}\mathrm{(AO)}_2\mathrm{KS}$ \cite{Liao2021High}
    which is an approximation of OMD.
    We exclude BOGD \cite{Zhao2012Fast},
    since its regret bound is same with FOGD
    and its performance is worse than FOGD.

    For all baseline algorithms,
    we tune the stepsize of gradient descent from $10^{[-3:1:3]}/\sqrt{T}$.
    For the other parameters,
    we follow the original papers.
    For POMDR,
    we set $B_0=\lceil 15\ln{T}\rceil$, $B=400$, $M=15$, $U=25$
    and multiply by a constant $c\in\{0.05,0.1\}$ on the learning rate $\lambda_t$
    (see \eqref{eq:AAAI2023:learning_rate}).
    Such values of parameters do not change our regret bound.
    For the ALD condition (see \eqref{eq:AAAI23:POMD:second_updating:ALD}),
    we set the threshold $10T^{-\zeta}$, $\zeta\in\{0.5,2/3\}$.
    For all algorithms,
    we tune the kernel parameter $\sigma\in 2^{[-2:1:6]}$.
    We randomly permutate the examples in the datasets 10 times
    and report the average results.
    All algorithms are implemented with R on a Windows machine
    with 2.8 GHz Core(TM) i7-1165G7 CPU
    \footnote{The codes: \url{https://github.com/JunfLi-TJU/OKL-Hinge}}.

\subsection{Experimental Results}

    Table \ref{tab:AAAI2023:POMDR:AMR}
    shows the average mistake ratio (AMR) and the average running time of all algorithms.
    The second column in Table \ref{tab:AAAI2023:POMDR:AMR}
    gives the budget size or number of random features (for FOGD).
    As a whole,
    our algorithm achieves the best prediction performance.
    The running time of our algorithm is also comparable with all baseline algorithms.
    Our algorithm performs better than FOGD, NOGD and SkeGD,
    since our algorithm uses OMD to update hypothesis and adaptive learning rates.
    Our algorithm performs better than $\mathrm{B}\mathrm{(AO)}_2\mathrm{KS}$.
    The reason is that $\mathrm{B}\mathrm{(AO)}_2\mathrm{KS}$ uses the restart technique
    when the budget exceeds $B$,
    while our algorithm only removes a half of examples
    which produces a better initial hypothesis than restart.
    The results verify \textbf{G} 2.

    \begin{table}[!t]
      \centering
      \begin{tabular}{l|c|r|r}
      \toprule
        \multirow{2}{*}{Algorithm} & \multirow{2}{*}{$B$}
        & \multicolumn{2}{c}{magic04,~$\sigma=0.5$, $\zeta=2/3$}\\
        \cline{3-4}&&AMR (\%)&Time (s) \\
      \toprule
        FOGD     & 400 & 16.88 $\pm$ 0.15 & 0.74 $\pm$ 0.01 \\
        NOGD     & 400 & 17.31 $\pm$ 0.20 & 1.57 $\pm$ 0.08 \\
        SkeGD    & 400 & 24.17 $\pm$ 0.95 & 1.07 $\pm$ 0.06 \\
        $\mathrm{B}\mathrm{(AO)}_2\mathrm{KS}$   & 400 & 22.04 $\pm$ 0.43 & 0.95 $\pm$ 0.05 \\
        POMDR    & 400 & \textbf{16.17} $\pm$ \textbf{0.21} & 0.91 $\pm$ 0.07 \\
      \toprule
        \multirow{2}{*}{Algorithm} & \multirow{2}{*}{$B$}
        & \multicolumn{2}{c}{w8a,~$\sigma=4$, $\zeta=0.5$}\\
        \cline{3-4}&&AMR (\%)&Time (s) \\
      \toprule
        FOGD     & 400 & 2.25 $\pm$ 0.04 & 8.29 $\pm$ 0.05 \\
        NOGD     & 400 & 2.89 $\pm$ 0.09 & 23.65 $\pm$ 0.17 \\
        SkeGD    & 400 & 3.04 $\pm$ 0.01 & 10.72 $\pm$ 0.15 \\
        $\mathrm{B}\mathrm{(AO)}_2\mathrm{KS}$   & 400 & 2.91 $\pm$ 0.02 & 14.08 $\pm$ 0.70 \\
        POMDR    & 400 & \textbf{2.12} $\pm$ \textbf{0.06} & 18.36 $\pm$ 0.11 \\
      \toprule
        \multirow{2}{*}{Algorithm} & \multirow{2}{*}{$B$}
        & \multicolumn{2}{c}{mushrooms,~$\sigma=2$, $\zeta=2/3$}\\
        \cline{3-4}&&AMR (\%)&Time (s) \\
      \toprule
        FOGD     & 400 & 0.31 $\pm$ 0.03 & 0.70 $\pm$ 0.02 \\
        NOGD     & 400 & 1.71 $\pm$ 0.33 & 1.87 $\pm$ 0.13 \\
        SkeGD    & 400 & 0.26 $\pm$ 0.02 & 0.51 $\pm$ 0.03 \\
        $\mathrm{B}\mathrm{(AO)}_2\mathrm{KS}$   & 400 & 1.67 $\pm$ 0.25 & 0.76 $\pm$ 0.05 \\
        POMDR    & 400 & \textbf{0.21} $\pm$ \textbf{0.03} & 0.44 $\pm$ 0.03 \\
      \bottomrule
      \end{tabular}
      \caption{Average online mistake ratio (AMR) of algorithms.}
      \label{tab:AAAI2023:POMDR:AMR}
    \end{table}

    Next we verify \textbf{G} 1.
    We change the kernel parameter, $\sigma$, and record $\mathcal{A}_T$ and $\bar{t}$.
    Table \ref{tab:AAAI2023:experiments_xi_T} gives the results.
    Theorem \ref{thm:AAAI2023:loss_regret_bound}
    proves our regret bound depends on $\sum^T_{\tau=1}\delta_{\tau} \leq 4(\mathcal{A}_T+D)$.
    $\mathcal{A}_T$ is a coarse approximation of $\sum^T_{\tau=1}\delta_{\tau}$.
    We use $\sum^T_{\tau=1}\delta_{\tau}$ as a proxy of $\mathcal{A}_T$
    and show it in Table \ref{tab:AAAI2023:experiments_xi_T}.
    We can find that if $\sigma$ is well tuned, then $\mathcal{A}_T\ll T$,
    such as the \textit{w8a} and \textit{mushrooms} datasets.
    In this case,
    our data-dependent regret bound is smaller than the worst-case bound.
    Besides,
    the last row of Table \ref{tab:AAAI2023:experiments_xi_T} gives the value of $\bar{t}$.
    If $\sigma$ is well tuned,
    then $\bar{t}=+\infty$ (see the \textbf{Initialization} of Algorithm \ref{alg:AAAI2023:POMDR}).
    In this case,
    our algorithm only executes POMD and the size of budget is smaller than $B_0=\lceil 15\ln{T}\rceil$.
    The results verify \textbf{G} 1.
    It is worth mentioning that
    the kernel function whose eigenvalues decay exponentially may not be good.
    The third row in Table \ref{tab:AAAI2023:experiments_xi_T} verifies this claim.
    How to choose a good kernel function is left to future work.

    \begin{table}[!t]
      \centering
      \begin{tabular}{l|rr|rr|rr}
      \toprule
       -& \multicolumn{2}{c|}{magic04}  &\multicolumn{2}{c|}{w8a} &\multicolumn{2}{c}{mushrooms}\\
       \hline
      $\sigma$& 0.5 & 4 & 4 & 64 & 0.5 & 2 \\
      \toprule
      AMR  &\textbf{16.17}   & 23.26  &\textbf{2.12}     &3.12  &0.83    &\textbf{0.21}   \\
      $\mathcal{A}_T$&6385   & 10337  &2594    &4256& 8122    &68   \\
      $\bar{t}$      &469  & $+\infty$  &3582    &$+\infty$& 135    &$+\infty$   \\
      \bottomrule
      \end{tabular}
      \caption{The values of $\mathcal{A}_T$ and $\bar{t}$ on different $\sigma$.}
      \label{tab:AAAI2023:experiments_xi_T}
    \end{table}

\section{Conclusion}

    In this paper,
    we proposed a new online kernel learning algorithm
    and improved the previous kernel alignment regret bound
    and computational complexity simultaneously.
    Our algorithm combines the OMD framework and the ALD condition,
    and also invents a new budget maintaining approach.
    Our main theoretical contribution is that
    we proved the size of budget maintained by the ALD condition.
    Our result does not require the examples satisfying i.i.d.
    and implies further application of the ALD condition on online kernel learning.

\section*{Acknowledgements}

    This work is supported by
    the National Natural Science Foundation of China under grants No. 62076181.
    We thank all anonymous reviewers for their valuable comments and suggestions,
    especially for bringing \citet{HornMatrix2012Matrix} to our attention.

\bibliography{aaai23}

\newpage

\appendix

\section{Basic Lemmas}

    \begin{lemma}
    \label{lemma:AAAI23:auxilimary_lemma_1}
        Let $\Omega$ be a convex compact set,
        $\psi$ be a convex function on $\Omega$,
        $f'_{t-1}$ be an arbitrary point in $\Omega$,
        and $x\in\Omega$.
        If
        $$
            f^\ast=\mathop{\arg\min}_{f\in\Omega}
            \left\{\langle x,f\rangle+\mathcal{D}_{\psi}(f,f'_{t-1})\right\},
        $$
        then for any $u\in\Omega$,
        $$
            \langle f^\ast-u,x\rangle
            \leq \mathcal{D}_{\psi}(u,f'_{t-1})-\mathcal{D}_{\psi}(u,f^\ast)
            -\mathcal{D}_{\psi}(f^\ast,f'_{t-1}).
        $$
    \end{lemma}
    \begin{proof}[Proof of Lemma \ref{lemma:AAAI23:auxilimary_lemma_1}]
        This lemma is a direct generalization of Proposition 18 in \cite{Chiang2012Online}.
        Thus we omit the proof.
    \end{proof}

    \begin{lemma}
    \label{lemma:AAAI23:auxilimary_lemma_2}
        Let $\sigma_t\geq 0$ for all $t=1,\ldots,T$
        and $\varepsilon=\max_{t=1,\ldots,T}\sigma_t$. Then
        $$
            \sum^T_{t=1}\frac{\sigma_t}{\sqrt{\varepsilon+\sum^{t-1}_{\tau=1}\sigma_\tau}}
            \leq 2\sqrt{2}\sqrt{\sum^{T}_{t=1}\sigma_t}.
        $$
    \end{lemma}
    \begin{proof}[Proof of Lemma \ref{lemma:AAAI23:auxilimary_lemma_2}]
    Denote by $\sigma_0=\varepsilon$.
    We decompose the term as follows
    \begin{align*}
        &\sum^T_{t=1}\frac{\sigma_t}{\sqrt{\varepsilon+\sum^{t-1}_{\tau=1}\sigma_\tau}}\\
        =&\sum^T_{t=1}\frac{\sigma_t}{\sqrt{\sum^{t-1}_{\tau=0}\sigma_\tau}}\\
        =&\frac{\sigma_1}{\sqrt{\varepsilon}}
        +\sum^T_{t=2}\frac{\sigma_{t-1}}{\sqrt{\varepsilon+\sum^{t-1}_{\tau=1}\sigma_\tau}}
        +\sum^T_{t=2}\frac{\sigma_t-\sigma_{t-1}}{\sqrt{\varepsilon+\sum^{t-1}_{\tau=1}\sigma_\tau}}.
    \end{align*}
    We first analyze the third term.
    \begin{align*}
        &\sum^T_{t=2}\frac{\sigma_{t}-\sigma_{t-1}}{\sqrt{\varepsilon+\sum^{t-1}_{\tau=1}\sigma_\tau}}\\
        =&\frac{-\sigma_{1}}{\sqrt{\varepsilon+\sigma_1}}
        +\frac{\sigma_{T}}{\sqrt{\varepsilon+\sum^{T-1}_{\tau=1}\sigma_\tau}}+\\
        &\sum^{T-1}_{t=2}\sigma_t\left[\frac{1}{\sqrt{\varepsilon+\sum^{t-1}_{\tau=1}\sigma_\tau}}
        -\frac{1}{\sqrt{\varepsilon+\sum^{t}_{\tau=1}\sigma_\tau}}\right]\\
        \leq&-\frac{\sigma_{1}}{\sqrt{\varepsilon+\sigma_1}}
        +\max_{t=1,\ldots,T}\sigma_t\cdot\frac{1}{\sqrt{\varepsilon+\sigma_1}}.
    \end{align*}
    Now we analyze the second term.
    \begin{align*}
        \sum^T_{t=2}\frac{\sigma_{t-1}}{\sqrt{\varepsilon+\sum^{t-1}_{\tau=1}\sigma_\tau}}
        =\frac{\sigma_{1}}{\sqrt{\varepsilon+\sigma_1}}
        +\sum^T_{t=3}\frac{\sigma_{t-1}}{\sqrt{\varepsilon+\sum^{t-1}_{\tau=1}\sigma_\tau}}.
    \end{align*}
        For any $a> 0$ and $b> 0$,
        we have $2\sqrt{a}\sqrt{b}\leq a+b$.
        Let $a=\varepsilon+\sum^{t-1}_{\tau=1}\sigma_\tau$ and
            $b=\varepsilon+\sum^{t-2}_{\tau=1}\sigma_\tau$.
        Then we have
        \begin{align*}
            2\sqrt{\varepsilon+\sum^{t-1}_{\tau=1}\sigma_\tau}
            \cdot\sqrt{\varepsilon+\sum^{t-2}_{\tau=1}\sigma_\tau}
            \leq 2\left(\varepsilon+\sum^{t-1}_{\tau=1}\sigma_\tau\right)-\sigma_{t-1}.
        \end{align*}
        Dividing by $\sqrt{a}$ and rearranging terms yields
        \begin{align*}
            \frac{1}{2}\frac{\sigma_{t-1}}{\sqrt{\varepsilon+\sum^{t-1}_{\tau=1}\sigma_\tau}}
            \leq \sqrt{\varepsilon+\sum^{t-1}_{\tau=1}\sigma_\tau}
            -\sqrt{\varepsilon+\sum^{t-2}_{\tau=1}\sigma_\tau}.
        \end{align*}
        Summing over $t=3,\ldots,T$,
        we obtain
        \begin{align*}
            &\sum^T_{t=3}\frac{\sigma_{t-1}}{\sqrt{\varepsilon+\sum^{t-1}_{\tau=1}\sigma_\tau}}
            \leq 2\sqrt{\varepsilon+\sum^{T-1}_{\tau=1}\sigma_\tau}
            -2\sqrt{\varepsilon+\sigma_1}.
        \end{align*}
        Summing over all results, we have
        \begin{align*}
            \sum^T_{t=1}\frac{\sigma_t}{\sqrt{\varepsilon+\sum^{t-1}_{\tau=1}\sigma_\tau}}
            \leq &2\sqrt{\varepsilon+\sum^{T-1}_{\tau=1}\sigma_\tau}
            -2\sqrt{\varepsilon+\sigma_1}+\\
            &\max_t\sigma_t\cdot\frac{1}{\sqrt{\varepsilon+\sigma_1}}+\frac{\sigma_1}{\sqrt{\varepsilon}}\\
            \leq&2\sqrt{\max_{t=1,\ldots,T}\sigma_t+\sum^{T}_{\tau=1}\sigma_\tau}\\
            \leq&2\sqrt{2}\sqrt{\sum^{T}_{\tau=1}\sigma_\tau},
        \end{align*}
        which concludes the proof.
        \end{proof}
    \begin{lemma}
    \label{lemma:AAAI23:auxilimary_lemma:path-length}
        Let $\nabla_t=-y_t\kappa(\mathbf{x}_t,\cdot)$ and
        $\bar{\nabla}_t=-\sum^{M_t}_{r=1}\frac{y_{t-r}}{M_t}\kappa(\mathbf{x}_{t-r},\cdot)$.
        Let $\mathcal{A}_T=\sum^T_{t=1}\kappa(\mathbf{x}_t,\mathbf{x}_t)
        -\frac{1}{T}{\bf Y}^\top_T{\bf K}_T{\bf Y}_T$.
        For any $t>1$,
        we have
        $$
            \sum^T_{\tau=t}(0\wedge[\Vert {\nabla}_{\tau}-\bar{\nabla}_{\tau}\Vert^2_{2}
            -\Vert \bar{\nabla}_{\tau}\Vert^2_{2}])
            \leq 4\mathcal{A}_T+7D.
        $$
    \end{lemma}
    \begin{proof}[Proof of Lemma \ref{lemma:AAAI23:auxilimary_lemma:path-length}]
        Let $\mu_T=-\frac{1}{T}\sum^T_{\tau=1}y_{\tau}\kappa(\mathbf{x}_{\tau},\cdot)$.
        Let $E_t=\{2\leq \tau\leq t:\nabla_\tau\neq 0\}$.
        We consider two cases.\\
        \textbf{Case 1}: $t\leq M$.\\
        In this case, we have
        \begin{align*}
            &\sum^M_{\tau=1}(0\wedge[\Vert {\nabla}_{\tau}-\bar{\nabla}_{\tau}\Vert^2_{\mathcal{H}}
            -\Vert \bar{\nabla}_{\tau}\Vert^2_{\mathcal{H}}])\\
            \leq& \Vert {\nabla}_{1}\Vert^2_{\mathcal{H}}
            +\sum^M_{\tau=2}\Vert {\nabla}_{\tau}-\mu_\tau+\mu_\tau-\bar{\nabla}_{\tau}\Vert^2_{\mathcal{H}}
            \cdot\mathbb{I}_{\nabla_t\neq 0}\\
            \leq& D
            +2\underbrace{\sum_{\tau\in E_M}\Vert {\nabla}_{\tau}-\mu_\tau\Vert^2_{\mathcal{H}}}_{\Xi_1}
            +2\underbrace{\sum_{\tau\in E_M}\Vert\mu_\tau-\bar{\nabla}_{\tau}\Vert^2_{\mathcal{H}}}_{\Xi_2}.
        \end{align*}
        \textbf{Case 2}: $t>M$.
        We have
        \begin{align*}
            &\sum^T_{\tau=M+1}(0\wedge[\Vert {\nabla}_{\tau}-\bar{\nabla}_{\tau}\Vert^2_{\mathcal{H}}
            -\Vert \bar{\nabla}_{\tau}\Vert^2_{\mathcal{H}}])\cdot\mathbb{I}_{\nabla_t\neq 0}\\
            \leq& 2\sum_{\tau\in E_T\setminus E_M}\Vert {\nabla}_{\tau}-\mu_T\Vert^2_{\mathcal{H}}
            +2\underbrace{\sum_{\tau\in E_T\setminus E_M}
            \Vert\mu_T-\bar{\nabla}_{\tau}\Vert^2_{\mathcal{H}}}_{\Xi_3}.
        \end{align*}
        For $\Xi_1$, we have
        \begin{align*}
            \Xi_1
            \leq\sum_{t\in E_M}\Vert {\nabla}_{t}-\mu_T\Vert^2_{\mathcal{H}}
            \leq\sum^M_{t=2}\Vert y_t\kappa(\mathbf{x}_t,\cdot)+\mu_T\Vert^2_{\mathcal{H}}.
        \end{align*}
        For $\Xi_2$, we have
        \begin{align*}
            \Xi_2
            \leq&\sum_{t\in E_M}\Vert\frac{1}{t}\sum^t_{s=1}y_s\kappa(\mathbf{x}_s,\cdot)
            -\frac{1}{t-1}y_s\sum^{t-1}_{s=1}\kappa(\mathbf{x}_s,\cdot)\Vert^2_{\mathcal{H}}\\
             \leq&\sum_{t\in E_M}\Vert\frac{1}{t}y_t\kappa(\mathbf{x}_t,\cdot)+
            \left(\frac{1}{t}-\frac{1}{t-1}\right)y_s\sum^{t-1}_{s=1}\kappa(\mathbf{x}_s,\cdot)\Vert^2_{\mathcal{H}}\\
            \leq&\sum^M_{t=2}\frac{2}{t^2}D
            +\frac{2}{t^2(t-1)^2}
            \sum_{t\in E_M}\Vert y_s\sum^{t-1}_{s=1}\kappa(\mathbf{x}_s,\cdot)\Vert^2_{\mathcal{H}}\\
            \leq&3D.
        \end{align*}
        Next, we analyze $\Xi_3$.
        \begin{align*}
            \Xi_4
            =&\sum_{t\in E_T\setminus E_M}\left\Vert\mu_T+
            \frac{1}{M}\sum^M_{r=1}y_{t-r}\kappa(\mathbf{x}_{t-r},\cdot)\right\Vert^2_{\mathcal{H}}\\
            \leq&\frac{1}{M}\sum^M_{r=1}\sum^T_{t=M+1}\left\Vert\mu_T+
            y_{t-r}\kappa(\mathbf{x}_{t-r},\cdot)\right\Vert^2_{\mathcal{H}}\\
            \leq&\mathcal{A}_T.
        \end{align*}
        Combining all results concludes the proof.
    \end{proof}

\section{Proof of Theorem \ref{thm:AAAI2023:size_budget}}

    \begin{proof}[Proof of Theorem \ref{thm:AAAI2023:size_budget}]
        We first state a key lemma
        which proves that the eigenvalues of a Hermitian matrix
        ${\bf A}\in\mathbb{C}^{T\times T}$
        are interlaced with those of any principal submatrix
        ${\bf B}\in\mathbb{C}^{m\times m}$, $m\leq T-1$.

        \begin{lemma}[Theorem 4.3.28 in \citet{HornMatrix2012Matrix}]
        \label{lemma:AAAI2023:Hwang2004Cauchy}
            Let ${\bf A}$ be a Hermitian matrix of order $T$,
            and let ${\bf B}$ be a principle submatrix of ${\bf A}$ of order $m$.
            If $\lambda_T\leq \lambda_{T-1}\leq\ldots\leq\lambda_2\leq\lambda_1$
            lists the eigenvalues of ${\bf A}$
            and $\mu_m\leq \mu_{m-1}\leq\ldots\leq\mu_{1}$
            lists the eigenvalues of ${\bf B}$, then
            $$
                \lambda_{i+T-m}\leq \mu_i\leq\lambda_i,~i=1,2,\ldots,m.
            $$
        \end{lemma}
        The kernel matrix ${\bf K}_T$ is positive semidefinite and real symmetric.
        Thus ${\bf K}_T$ satisfies Theorem \ref{lemma:AAAI2023:Hwang2004Cauchy}.

        Observing that
        \begin{align*}
            \alpha_t
            =&\min_{{\bm \beta}\in\mathbb{R}^{\vert S_t\vert}}
            \left\Vert {\bm \Phi}_{S_t}{\bm \beta}-\kappa(\mathbf{x}_t,\cdot)\right\Vert^2_{\mathcal{H}}\\
            =&\kappa(\mathbf{x}_t,\mathbf{x}_t)-
            ({\bm \Phi}^\top_{S_t}\kappa(\mathbf{x}_t,\cdot))^\top
            {\bf K}^{-1}_{S_t}{\bm \Phi}^\top_{S_t}\kappa(\mathbf{x}_t,\cdot).
        \end{align*}
        The determinant of block matrix satisfies
        $$
            \mathrm{det}\left[
                          \begin{array}{cc}
                            {\bf D} & {\bf E} \\
                            {\bf E}^\top & c \\
                          \end{array}
                        \right]
                        =\mathrm{det}({\bf D})\cdot\mathrm{det}(c-{\bf E}^\top{\bf D}^{-1}{\bf E}).
        $$
        Let ${\bf D}={\bf K}_{S_t}$ and $c=\kappa(\mathbf{x}_t,\mathbf{x}_t)$.
        We have
        \begin{align*}
            &\frac{\mathrm{det}({\bf K}_{S_t\cup\{(\mathbf{x}_t,y_t)\}})}{\mathrm{det}({\bf K}_{S_t})}\\
            =&\kappa(\mathbf{x}_t,\mathbf{x}_t)-
            ({\bm \Phi}^\top_{S_t}\kappa(\mathbf{x}_t,\cdot))^\top
            {\bf K}^{-1}_{S_t}{\bm \Phi}^\top_{S_t}\kappa(\mathbf{x}_t,\cdot)
            =\alpha_t.
        \end{align*}
        If $\mathrm{ALD}_t$ does not hold,
        then $S_{t+1}=S_t\cup\{(\mathbf{x}_t,y_t)\}$.
        Let $\vert S_T\vert=k$ and $\{t_j\}^{k}_{j=1}$ be the set of time index at which
        $\mathrm{ALD}_{t_j}$ does not hold.
        It is obvious that ${\bf K}_{S_{t_j}}$ is a $j$-order principle submatrix of ${\bf K}_T$.
        Let $\{\lambda^{(j)}_i\}^{j}_{i=1}$ be the eigenvalues of ${\bf K}_{S_{t_j}}$,
        where $j=1,\ldots,k$.
        Let $\{\lambda_j\}^T_{j=1}$ be the eigenvalues of ${\bf K}_{T}$.
        Assuming that $\lambda_j\leq R_0r^{j}$.
        We first bound the value of $R_0$.
        Note that
        $$
            A\cdot T\leq\sum^T_{i=1}\kappa(\mathbf{x}_t,\mathbf{x}_t)
            =\sum^T_{j=1}\lambda_j\leq R_0\sum^T_{i=1}r^{i}.
        $$
        We have $R_0\geq \frac{AT}{\sum^T_{i=1}r^{i}}$.
        There must be a constant $C>0$ depending on $r$ and $D$ such that $R_0=CT$.

        We focus on the eigenvalues of ${\bf K}_{S_{t_k}}$,
        i.e., $\{\lambda^{(k)}_j\}^{k}_{j=1}$.
        Next we consider two cases.
        \begin{itemize}
          \item \textbf{Case 1}: $\lambda^{(k)}_k= R_0r^{k}$.\\
                In this case,
                we have $\lambda^{(k)}_j \geq \lambda^{(k)}_k\geq R_0r^{k}$ for $j=1,\ldots,k-1$.
                At $t=t_k$,
                we have
                $$
                    \frac{\mathrm{det}(\mathbf{K}_{S_{t_k}})}{\mathrm{det}(\mathbf{K}_{S_{t_{k-1}}})}
                    =\frac{\lambda^{(k)}_k\cdot\prod^{k-1}_{j=1}\lambda^{(k)}_j}{\prod^{k-1}_{j=1}\lambda^{(k-1)}_j}
                    >\alpha_t\geq DT^{-2}=:\alpha.
                $$
                Rearranging terms yields
                $$
                    \prod^{k-1}_{j=1}\lambda^{(k-1)}_j
                    <\frac{\lambda^{(k)}_k\cdot\prod^{k-1}_{j=1}\lambda^{(k)}_j}{\alpha}.
                $$
                Similarly, at $t=t_{k-1}$, we have
                $$
                    \prod^{k-2}_{j=1}\lambda^{(k-2)}_j
                    <\frac{\prod^{k-1}_{j=1}\lambda^{(k-1)}_j}{\alpha}
                    <\frac{\lambda^{(k)}_k\cdot\prod^{k-1}_{j=1}\lambda^{(k)}_j}{\alpha^2}.
                $$
                At $t=t_1$, we have
                $$
                    \lambda^{(1)}_1
                    <\frac{\lambda^{(k)}_k\cdot\prod^{k-1}_{j=1}\lambda^{(k)}_j}{\alpha^{k-1}}.
                $$
                Let $T=k$ and $m=1$ in
                Lemma \ref{lemma:AAAI2023:Hwang2004Cauchy}.
                We can obtain $\lambda^{(1)}_1\geq \lambda^{(k)}_k$.
                Thus it must be
                \begin{equation}
                \label{eq:AAAI23:proof_ALD:case_1_condtion}
                    \alpha^{k-1}<\prod^{k-1}_{j=1}\lambda^{(k)}_j.
                \end{equation}
                Lemma \ref{lemma:AAAI2023:Hwang2004Cauchy}
                gives $\lambda^{(k)}_j\leq \lambda_j, j=1,\ldots,k-1$.
                Under the assumption that $\{\lambda_j\}^T_{j=1}$
                decay exponentially,
                i.e., $\lambda_j \leq R_0r^{j}$,
                we have the following inequality
                $$
                    R^{k-1}_0r^{k(k-1)}< \prod^{k-1}_{j=1}\lambda^{(k)}_j
                    \leq \prod^{k-1}_{j=1}\lambda_j
                    \leq R^{k-1}_0r^{\frac{k(k-1)}{2}}.
                $$
                Let $\frac{k(k-1)}{2}\leq g(k)\leq k(k-1)$.
                According to \eqref{eq:AAAI23:proof_ALD:case_1_condtion},
                it is necessary to ensure
                $$
                    \alpha^{k-1}<R^{k-1}_0r^{g(k)}.
                $$
                Solving the inequality gives
                $$
                    k=\frac{2g(k)}{k-1}<2\frac{\ln{(\frac{R_0}{\alpha})}}{\ln{r^{-1}}}.
                $$
          \item \textbf{Case 2}: $\lambda^{(k)}_k<R_0r^{k}$.\\
                According to \textbf{Case 1},
                we have
                \begin{align*}
                    \frac{\mathrm{det}(\mathbf{K}_{S_{t_{k}}})}{\mathrm{det}(\mathbf{K}_{S_{t_{1}}})}
                    =&\frac{\mathrm{det}(\mathbf{K}_{S_{t_{k}}})}{\mathrm{det}(\mathbf{K}_{S_{t_{k-1}}})}
                    \cdot \frac{\mathrm{det}(\mathbf{K}_{S_{t_{k-1}}})}{\mathrm{det}(\mathbf{K}_{S_{t_{1}}})}\\
                    >&\alpha\cdot \frac{\mathrm{det}(\mathbf{K}_{S_{t_{k-1}}})}{\mathrm{det}(\mathbf{K}_{S_{t_{1}}})}
                    >\alpha^{k-1}.
                \end{align*}
                Using the definition of determinant,
                we obtain
                $$
                    \frac{\mathrm{det}(\mathbf{K}_{S_{t_{k}}})}{\mathrm{det}(\mathbf{K}_{S_{t_{1}}})}
                    =\frac{\lambda^{(k)}_k\cdot\prod^{k-1}_{j=1}\lambda^{(k)}_{j}}{\lambda^{(1)}_{1}}
                    >\alpha^{k-1},
                $$
                where $A\leq\lambda^{(1)}_{1}\leq D$.
                Rearranging terms gives
                \begin{align*}
                    \lambda^{(k)}_{k}
                    >&\frac{\lambda^{(1)}_{1}\alpha^{k-1}}{\prod^{k-1}_{j=1}\lambda^{(k)}_{j}}\\
                    \geq&A\alpha^{k-1}\cdot R^{-k+1}_0\cdot r^{-\frac{k(k-1)}{2}}\\
                    \geq&R_0r^{k}\cdot A\alpha^{k-1}\cdot R^{-k}_0r^{-k-\frac{k(k-1)}{2}},
                \end{align*}
                where we use the fact $\lambda^{(k)}_j\leq \lambda_j, j=1,\ldots,k-1$.
                Let
                $$
                    \alpha^{k-1}\cdot R^{-k}_0r^{-k-\frac{k(k-1)}{2}}>\frac{1}{A}.
                $$
                Rearranging terms yields
                $$
                    \frac{(k+2)(k-1)}{2}>\frac{k(k+1)}{2}
                    >\frac{\ln\left(A^{-1}\frac{R^{k}_0}{\alpha^{k-1}}\right)}{\ln{r^{-1}}}
                $$
                Solving the inequality gives
                \begin{align*}
                    k>\frac{2}{\ln{r^{-1}}}\ln\left(A^{-\frac{1}{k-1}}\frac{1}{\alpha}R^{\frac{k}{k-1}}_0\right)-2.
                \end{align*}
                In this case,
                we have $\lambda^{(k)}_{k}>R_0r^{k}$,
                which contradicts with the condition $\lambda^{(k)}_k<R_0r^{k}$.\\
                Thus it must be
                \begin{align*}
                    k
                    \leq&\frac{2}{\ln{r^{-1}}}\ln\left(A^{-\frac{1}{k-1}}\frac{R^{\frac{k}{k-1}}_0}{\alpha}\right)-2\\
                    =&\frac{2}{\ln{r^{-1}}}\ln\left(\left(\frac{R_0}{A}\right)^{-\frac{1}{k-1}}\frac{R_0}{\alpha}\right)-2\\
                    \leq& 2\frac{\ln\left(C_1\frac{R_0}{\alpha}\right)}{\ln{r^{-1}}}-2,
                \end{align*}
                where $C_1$ is any constant depending on $A$ such that $C_1\geq (R_0/A)^{-\frac{1}{k-1}}$.
        \end{itemize}
        Combining the two cases, we prove the first statement.\\
        If $\{\lambda_j\}^T_{j=1}$ decays polynomially,
        that is, there is a constant $R_0>0$ and $p\geq 1$,
        such that $\lambda_j\leq R_0j^{-p}$, $p\geq 1$.
        The proof is similar with previous analysis.
        We first bound the value of $R_0$.
        $$
            R_0\sum^T_{i=1}i^{-p}\geq
            \sum^T_{j=1}\lambda_j=\sum^T_{i=1}\kappa(\mathbf{x}_t,\mathbf{x}_t)\geq AT.
        $$
        We have $R_0\geq\frac{AT}{\sum^T_{i=1}i^{-p}}$.
        There must be a constant $C>0$ depending on $r$ and $D$ such that $R_0=CT$.
        \begin{itemize}
          \item \textbf{Case 1}: $\lambda^{(k)}_k= R_0k^{-p}$.\\
                In this case, we have
                $$
                    R^{k-1}_0k^{-p(k-1)}< \prod^{k-1}_{j=1}\lambda^{(k)}_j
                    \leq R^{k-1}_0\prod^{k-1}_{j=1}j^{-p}.
                $$
                According to \eqref{eq:AAAI23:proof_ALD:case_1_condtion},
                we obtain a necessary condition
                $$
                    \alpha^{k-1}<R^{k-1}_0\prod^{k-1}_{j=1}j^{-p}.
                $$
                Rearranging terms yields
                $$
                    ((k-1)!)^p \leq R^{k-1}_0\cdot(\alpha^{-1})^{k-1}.
                $$
                Using Stirling's formula, i.e., $k!\approx \sqrt{2\pi k}(k/\mathrm{e})^k$,
                we have
                $$
                    ((k-1)/\mathrm{e})^{p(k-1)}\leq R^{k-1}_0\cdot(\alpha^{-1})^{k-1}.
                $$
                Solving the inequality gives
                $$
                    k\leq \mathrm{e}\cdot R^{\frac{1}{p}}_0\cdot (\alpha^{-1})^{\frac{1}{p}}+1
                    \leq \mathrm{e}\cdot \left(\frac{R_0}{\alpha}\right)^{\frac{1}{p}}+1.
                $$
          \item \textbf{Case 2}: $\lambda^{(k)}_k<R_0k^{-p}$.\\
                In this case,
                \begin{align*}
                    \lambda^{(k)}_{k}
                    >&\frac{\lambda^{(1)}_{1}\alpha^{k-1}}{\prod^{k-1}_{j=1}\lambda^{(k)}_{j}}\\
                    \geq&A\alpha^{k-1}\cdot R^{-k+1}_0\cdot ((k-1)!)^p\\
                    =&R_0 k^{-p}\cdot A\alpha^{k-1}\cdot R^{-k}_0\cdot (k!)^p.
                \end{align*}
                Let
                $$
                    \alpha^{k-1}\cdot R^{-k}_0\cdot (k!)^p>\frac{1}{A}.
                $$
                Simplifying the inequality gives
                $$
                    (\frac{k}{\mathrm{e}})^{kp}>\frac{1}{A}R^{k}_0(\alpha^{-1})^{k-1}.
                $$
                Solving the inequality gives
                $$
                    k>\mathrm{e}\cdot A^{-\frac{1}{pk}}R^{\frac{1}{p}}_0(\alpha^{-1})^{\frac{k-1}{kp}}.
                $$
                In this case,
                $\lambda^{(k)}_{k}>R_0k^{-p}$,
                which contradicts with the condition $\lambda^{(k)}_{k}<R_0k^{-p}$.
                Thus
                $$
                    k\leq\mathrm{e}\cdot A^{-\frac{1}{pk}}R^{\frac{1}{p}}_0(\alpha^{-1})^{\frac{k-1}{kp}}.
                $$
                If $\alpha\leq 1$, then we have
                $$
                    k \leq \mathrm{e}\cdot A^{-\frac{1}{pk}} \left(\frac{R_0}{\alpha}\right)^{\frac{1}{p}}.
                $$
                If $1<\alpha\leq D$, then we have
                \begin{align*}
                    k
                    \leq&\mathrm{e}\cdot A^{-\frac{1}{pk}}R^{\frac{1}{p}}_0(\alpha^{-1})^{\frac{1}{p}}
                    \cdot (\alpha)^{\frac{1}{kp}}\\
                    \leq& \mathrm{e}\cdot \left(\frac{D}{A}\right)^{\frac{1}{pk}}\left(\frac{R_0}{\alpha}\right)^{\frac{1}{p}}\\
                    \leq& \mathrm{e}\cdot \left(C_2\frac{R_0}{\alpha}\right)^{\frac{1}{p}}.
                \end{align*}
                where $C_2$ is any constant such that $C_2\geq\left(\frac{D}{A}\right)^{\frac{1}{k}}$.
        \end{itemize}
        Up to now, we conclude the proof.
    \end{proof}

\section{Solution of OMD}

    Let $\psi_t(f)=\frac{1}{2\lambda_t}\Vert f\Vert^2_{\mathcal{H}}$.
    Then the Bregman divergence between any $f,g\in\mathcal{H}$ or $\mathbb{H}$ is
    $$
        \mathcal{D}_{\psi_t}(f,g)
        =\frac{\Vert f\Vert^2_{\mathcal{H}}}{2\lambda_t}
        -\frac{\Vert g\Vert^2_{\mathcal{H}}}{2\lambda_t}
        -\frac{1}{\lambda_t}\langle g,f-g\rangle_{\mathcal{H}}
        =\frac{\Vert f-g\Vert^2_{\mathcal{H}}}{2\lambda_t}.
    $$
    The two step update of OMD adopted by our algorithm is
    \begin{align*}
        f_t=&\mathop{\arg\min}_{f\in\mathcal{H}}
                \left\{\langle f,\hat{\nabla}_t\rangle+\mathcal{D}_{\psi_t}(f,f'_{t-1})\right\},\\
        f'_t=&\mathop{\arg\min}_{f\in\mathbb{H}}
                \left\{\langle f,\tilde{\nabla}_{t}\rangle+\mathcal{D}_{\psi_t}(f,f'_{t-1})\right\}.
    \end{align*}
    Using the Lagrangian multiplier method,
    it is easy to obtain
    \begin{align*}
        &f_{t}=f'_{t-1}-\lambda_t\bar{\nabla}_{t},\\
        &\bar{f}'_{t}=f'_{t-1}-\lambda_t\tilde{\nabla}_{t},\quad
        f'_{t}=\min\left\{1,\frac{U}{\Vert \bar{f}'_{t}\Vert_{\mathcal{H}}}\right\}\bar{f}'_{t}.
    \end{align*}

\section{Proof of Theorem \ref{thm:AAAI2023:loss_regret_bound}}

    Recalling that $\bar{t}=\min_{t\in [T]}\{t:\vert S_t\vert=B_0\}$,
    where $B_0>2\frac{\ln{(\frac{C_1R_0}{\alpha})}}{\ln{r^{-1}}}$.
    If the eigenvalues of ${\bf K}_T$ decay exponentially,
    then Theorem 1 proves that
    $\vert S_T\vert<2\frac{\ln{(\frac{C_1R_0}{\alpha})}}{\ln{r^{-1}}}$.
    In this case,
    $\bar{t}\geq T$ and our algorithm always execute POMD (i.e., do not execute OMDR).
    If the eigenvalues of ${\bf K}_T$ decay polynomially,
    then Theorem 1 proves that
    $\vert S_T\vert<\mathrm{e}(\frac{C_2R_0}{\alpha})^{\frac{1}{p}}$.
    In this case,
    it may be $\bar{t}<T$ and our algorithm will execute POMD and OMDR.
    Next we consider two cases.\\
    \textbf{Case 1}: $\bar{t} \geq T$.\\
        Using the convexity of the hinge loss function,
        we have
        \begin{align*}
            &\sum^{T}_{t=1}\ell(f_{t}(\mathbf{x}_t),y_t)-\sum^{T}_{t=1}\ell(f(\mathbf{x}_t),y_t)\\
            \leq& \sum^{T}_{t=1}\langle f_{t}-f, \tilde{\nabla}_{t}\rangle
            +\sum^{T}_{t=1}\underbrace{\langle f_{t}-f, \nabla_{t}-\tilde{\nabla}_{t}\rangle}_{\Xi_1}\\
            =&\sum^{T}_{t=1}
            \underbrace{\langle f_{t}-f'_{t},\bar{\nabla}_{t}\rangle}_{\mathcal{T}_{2,1,t}}
            +\underbrace{\langle f'_{t}-f,\tilde{\nabla}_{t}\rangle}_{\mathcal{T}_{2,2,t}}
            +\underbrace{\langle f_{t}-f'_{t},\tilde{\nabla}_{t}-\bar{\nabla}_{t}\rangle}_{\mathcal{T}_{2,3,t}}
            +\Xi_1.
        \end{align*}
        Using Lemma \ref{lemma:AAAI23:auxilimary_lemma_1},
        let $\Omega=\mathcal{H}$, $x=\bar{\nabla}_t$ and $u=f'_t$,
        we can simplify $\mathcal{T}_{2,1,t}$.
        Besides,
        let $\Omega=\mathbb{H}$, $x=\tilde{\nabla}_t$ and $u=f$,
        we can simplify $\mathcal{T}_{2,2,t}$.
        For $\mathcal{T}_{2,1,t}$ and $\mathcal{T}_{2,2,t}$,
        we have
        \begin{align}
            \mathcal{T}_{2,1,t}
            \leq& \mathcal{D}_{\psi_{t}}(f'_{t},f'_{t-1})
            -\mathcal{D}_{\psi_{t}}(f'_{t},f_{t})-\mathcal{D}_{\psi_{t}}(f_{t},f'_{t-1}),
            \label{eq:ECML2021:BOKS:intermediate_result_T_2:Bregman_divergence_1}\\
            \mathcal{T}_{2,2,t}
            \leq& \mathcal{D}_{\psi_{t}}(f,f'_{t-1})
            -\mathcal{D}_{\psi_{t}}(f,f'_{t})-\mathcal{D}_{\psi_{t}}(f'_{t},f'_{t-1}).
            \label{eq:ECML2021:BOKS:intermediate_result_T_2:Bregman_divergence_2}
        \end{align}
        Combining the above two inequalities,
        we obtain
        \begin{align*}
            \sum^{T}_{t=1}[\mathcal{T}_{2,1,t}+\mathcal{T}_{2,2,t}]
            \leq& \sum^{T}_{t=1}\underbrace{[\mathcal{D}_{\psi_{t}}(f,f'_{t-1})
            -\mathcal{D}_{\psi_{t}}(f,f'_{t})]}_{\Xi_2}-\\
            &\sum^{T}_{t=1}\underbrace{[\mathcal{D}_{\psi_{t}}(f'_{t},f_{t})
            +\mathcal{D}_{\psi_{t}}(f_{t},f'_{t-1})]}_{\Xi_3}.
        \end{align*}
        We first analyze $\Xi_2$.\\
        Recalling that $\{\lambda_t\}$ is non-increasing.
        \begin{align*}
            \sum^{T}_{t=1}\Xi_2
            =&\sum^{T}_{t=1}\frac{\Vert f-f'_{t-1}\Vert^2_{\mathcal{H}}}{2\lambda_t}
            -\frac{\Vert f-f'_{t}\Vert^2_{\mathcal{H}}}{2\lambda_t}\\
            \leq&\frac{\Vert f-f'_{0}\Vert^2_{\mathcal{H}}}{2\lambda_{1}}
            +\sum^{T-1}_{t=1}\frac{\Vert f-f'_{t}\Vert^2_{\mathcal{H}}}{2}
            \left[\frac{1}{\lambda_{t+1}}-\frac{1}{\lambda_{t}}\right]\\
            \leq&\frac{2U^2}{\lambda_{T}},
        \end{align*}
        where $\Vert f'_t\Vert_{\mathcal{H}}\leq U$ for all $t\geq 1$.\\
        Next we analyze $\mathcal{T}_{2,3,t}-\Xi_3$.
        \begin{align*}
            &\left\langle f_{t}-f'_{t},\tilde{\nabla}_{t}-\bar{\nabla}_{t}\right\rangle
            -[\mathcal{D}_{\psi_{t}}(f'_{t},f_{t})+\mathcal{D}_{\psi_{t}}(f_{t},f'_{t-1})]\\
            =&\left\langle f_{t}-f'_{t},\tilde{\nabla}_{t}-\bar{\nabla}_{t}\right\rangle
            -\frac{\Vert f'_t-f_t\Vert^2_{\mathcal{H}}}{2\lambda_t}
            -\frac{\Vert f_t-f'_{t-1}\Vert^2_{\mathcal{H}}}{2\lambda_t}\\
            =&\frac{\lambda_t}{2}\Vert\tilde{\nabla}_{t}-\bar{\nabla}_{t}\Vert^2_{\mathcal{H}}
            -\frac{1}{2\lambda_t}\Vert f_t-f'_t-\lambda_t(\tilde{\nabla}_{t}-\bar{\nabla}_{t})\Vert^2_{\mathcal{H}}-\\
            &\frac{1}{2\lambda_t}\Vert f_t-f'_{t-1}\Vert^2_{\mathcal{H}}\\
            \leq&\frac{\lambda_t}{2}
            \left[\Vert\tilde{\nabla}_{t}-\bar{\nabla}_{t}\Vert^2_{\mathcal{H}}
            -\Vert \bar{\nabla}_t\Vert^2_{\mathcal{H}}\right].
        \end{align*}
        Furthermore, we can obtain
        \begin{align*}
            &\sum^{T}_{t=1}[\Xi_2+\mathcal{T}_{2,3,t}-\Xi_3]\\
            \leq&\frac{2U^2}{\lambda_{T}}
            +\sum^{T}_{t=1}\frac{\lambda_t}{2}[\Vert \tilde{\nabla}_{t}-\bar{\nabla}_{t}\Vert^2_{\mathcal{H}}-
            \Vert \bar{\nabla}_t\Vert^2_{\mathcal{H}}]\\
            \leq&3U\sqrt{3+
            \sum^{T}_{t=1}([\Vert \tilde{\nabla}_{t}-\bar{\nabla}_{t}\Vert^2_{\mathcal{H}}-
            \Vert \bar{\nabla}_t\Vert^2_{\mathcal{H}}]\wedge 0)},
        \end{align*}
        in which
        \begin{align*}
            &\Vert \tilde{\nabla}_t-\bar{\nabla}_t\Vert^2_{\mathcal{H}}-
            \Vert \bar{\nabla}_t\Vert^2_{\mathcal{H}}\\
            =&\Vert \tilde{\nabla}_t-\nabla_t+\nabla_t-\bar{\nabla}_t\Vert^2_{\mathcal{H}}-
            \Vert \bar{\nabla}_t\Vert^2_{\mathcal{H}}\\
            =&
            [\Vert\nabla_t-\bar{\nabla}_t\Vert^2_{\mathcal{H}}-
            \Vert \bar{\nabla}_t\Vert^2_{\mathcal{H}}]+\\
            &\langle \tilde{\nabla}_t-\nabla_t,\tilde{\nabla}_t+\nabla_t-2\bar{\nabla}_t\rangle\\
            =&
            \Vert\nabla_t-\bar{\nabla}_t\Vert^2_{\mathcal{H}}-
            \Vert \bar{\nabla}_t\Vert^2_{\mathcal{H}}+
            \langle \tilde{\nabla}_t-\nabla_t,\nabla_t-2\bar{\nabla}_t\rangle.
        \end{align*}
        The last equality comes from
        $$
            \langle \tilde{\nabla}_t,\nabla_{t}-\tilde{\nabla}_t\rangle=0.
        $$
        The reason is that $\tilde{\nabla}_t\in\mathcal{H}_{S_t}$
        and $(\nabla_{t}-\tilde{\nabla}_t)$ is orthogonal with the element in $\mathcal{H}_{S_t}$.
        We further obtain
        $$
            \sum^{T}_{t=1}(\langle \tilde{\nabla}_{t}-\nabla_t,\nabla_t-2\bar{\nabla}_{t}\rangle \wedge0)
            \leq 3\sum^{T}_{t=1}\sqrt{\alpha_t}\mathbb{I}_{\nabla_t\neq 0,\mathrm{ALD}_t~\mathrm{holds}}.
        $$
        Next we analyze $\Xi_1$.
        Recalling that
        \begin{align*}
            \mathrm{ALD}_t:\sqrt{\alpha_t} \leq T^{-\zeta}.
        \end{align*}
        We further obtain
        \begin{align*}
            \sum^{T}_{t=1}\Xi_1
            =&\sum^{T}_{t=1}
            \left\langle f_t-f,\nabla_{t}-\tilde{\nabla}_t\right\rangle\\
            =&\sum^{T}_{t=1}\left\langle f,\nabla_{t}-\tilde{\nabla}_t\right\rangle
            \cdot\mathbb{I}_{\nabla_t\neq 0,\mathrm{ALD}_t~\mathrm{holds}}\\
            \leq&\sum^{T}_{t=1}U\sqrt{\alpha_t}
            \cdot\mathbb{I}_{\nabla_t\neq 0,\mathrm{ALD}_t~\mathrm{holds}}\\
            \leq&\sum_{t\in E_T}U\frac{1}{T^{\zeta}} \leq UT^{1-\zeta},
        \end{align*}
        where the second equality comes from
        $$
            \langle f'_{t-1},\nabla_{t}-\tilde{\nabla}_t\rangle=0,\quad
            f_t=f'_{t-1}+\lambda_ty_t\tilde{\nabla}_t.
        $$
        The reason is that $f'_{t-1}\in\mathcal{H}_{S_t}$
        and $(\nabla_{t}-\tilde{\nabla}_t)$ is orthogonal with the element in $\mathcal{H}_{S_t}$.\\
        Let $\zeta=1$.
        Summing over all results, we obtain
        \begin{align*}
            \mathrm{Reg}(f)\leq 3U\sqrt{6+\sum^T_{\tau=1}\delta_{\tau}}+U
            \leq 3U\sqrt{\sum^T_{\tau=1}\delta_{\tau}}+9U.
        \end{align*}
    \textbf{Case 2}: $\bar{t} < T$.\\
    POMDR executes OMDR from $t>\bar{t}$.
    Assuming that OMDR restarts $N$ times
    and the corresponding time indexs are $\{t_i\}^N_{i=1}$.
    The time interval $[\bar{t}+1,T]$ can be divided as follows
    $$
        [\bar{t}+1,T]=[\bar{t}+1,t_1]\cup^{N-1}_{j=1} [t_j+1,t_{j+1}]\cup [t_N+1,T].
    $$
    For $t\leq t_1$, the regret bound is same with \textbf{Case 1}, i.e.,
    $$
        \sum^{t_1}_{t=1}[\ell(f_{t}(\mathbf{x}_t),y_t)-\ell(f(\mathbf{x}_t),y_t)]
        \leq 3U\sqrt{\sum^{t_1}_{\tau=1}\delta_{\tau}}+9U.
    $$
    Next we analyze the regret in the time interval $[t_1+1,T]$.
    We analyze the regret in a fixed time interval $[t_j+1,t_{j+1}]$
    where $j=1,\ldots,N$.
    We define $t_{N+1}=T$.
    Similar with previous analysis,
    we have
    \begin{align*}
        &\sum^{t_{j+1}}_{t=t_j+1}\ell(f_{t}(\mathbf{x}_t),y_t)
        -\sum^{t_{j+1}}_{t=t_j+1}\ell(f(\mathbf{x}_t),y_t)\leq\\
        & \sum^{t_{j+1}}_{t=t_j+1}
        [\underbrace{\langle f_{t}-f'_{t},\bar{\nabla}_{t}\rangle}_{\mathcal{T}_{2,1,t}}
        +\underbrace{\langle f'_{t}-f,{\nabla}_{t}\rangle}_{\mathcal{T}_{2,2,t}}
        +\underbrace{\langle f_{t}-f'_{t},{\nabla}_{t}-\bar{\nabla}_{t}\rangle}_{\mathcal{T}_{2,3,t}}],
    \end{align*}
    where
    \begin{align*}
        &\sum^{t_{j+1}}_{t=t_j+1}[\mathcal{T}_{2,1,t}+\mathcal{T}_{2,2,t}]\\
        \leq& \sum^{t_{j+1}}_{t=t_j+1}[\mathcal{D}_{\psi_t}(f,f'_{t-1})
        -\mathcal{D}_{\psi_t}(f,f'_t)]-\\
        &\sum^{t_{j+1}}_{t=t_j+1}[\mathcal{D}_{\psi_{t}}(f'_t,f_t)
        +\mathcal{D}_{\psi_t}(f_{t},f'_{t-1})]\\
        \leq&\frac{2U^2}{\lambda_{t_{j+1}}}
        -\sum^{t_{j+1}}_{t=t_j+1}[\mathcal{D}_{\psi_{t}}(f'_t,f_t)
        +\mathcal{D}_{\psi_t}(f_{t},f'_{t-1})]
    \end{align*}
    We have
    \begin{align*}
        &\langle f_{t}-f'_{t},{\nabla}_{t}-\bar{\nabla}_{t}\rangle
        -\mathcal{D}_{\psi_{t}}(f'_t,f_t)+\mathcal{D}_{\psi_t}(f_{t},f'_{t-1})\\
    \leq& \frac{\lambda_t}{2}
            \left[\Vert\nabla_{t}-\bar{\nabla}_{t}\Vert^2_{\mathcal{H}}
            -\Vert \bar{\nabla}_t\Vert^2_{\mathcal{H}}\right]
            =\frac{\lambda_t}{2}\delta_t.
    \end{align*}
    Recalling that the learning rates is defined by
    $$
        \lambda_t= \frac{U}{\sqrt{\varepsilon+\sum^{t-1}_{\tau=t_j+1}\delta_{\tau}}},~
        \varepsilon= \max_{t_j+1\leq t\leq t_{j+1}+1}\sigma_t.
    $$
    Combining all,
    we obtain
    $$
        \mathrm{Reg}_j(f)\leq 3U{\sqrt{\sum^{t_{j+1}}_{\tau=t_j+1}\delta_{\tau}}}.
    $$
    Next we need to bound $N$.\\
    It is obvious that $t_1\geq B$
    and $t_{t+1}-t_j \geq \frac{B}{2}$ for $j\geq 1$.
    Using the following relation,
    $$
        B + \frac{B}{2}(N-1)\leq t_1 +\sum^{N+1}_{j=1}(t_{j+1}-t_j)=T,
    $$
    we obtain $N\leq \frac{2T}{B}-1$.\\
    Combining the regret in $N$ time intervals gives
    \begin{align*}
        &\sum^{T}_{t=t_1+1}\ell(f_{t}(\mathbf{x}_t),y_t)
        -\sum^{T}_{t=t_1+1}\ell(f(\mathbf{x}_t),y_t)\\
        =&\sum^N_{j=1}\mathrm{Reg}_j(f)\\
        =&\sum^N_{j=1}3U{\sqrt{\sum^{t_{j+1}}_{\tau=t_j+1}\delta_{\tau}}}\\
        =&3{\sqrt{N\sum^N_{j=1}\sum^{t_{j+1}}_{\tau=t_j+1}\delta_{\tau}}}\\
        \leq&3\frac{\sqrt{2T}{\sqrt{\sum^T_{\tau=t_1+1}\delta_{\tau}}}}{\sqrt{B}}.
    \end{align*}
    Combing with the regret in $[1,\bar{t}]$, we have
    $$
        \mathrm{Reg}(f)
        \leq 3U\sqrt{\sum^{T}_{\tau=1}\delta_{\tau}}+9U
        +3\sqrt{2}\frac{\sqrt{T}{\sqrt{\sum^T_{\tau=1}\delta_{\tau}}}}{\sqrt{B}}.
    $$

\section{Proof of Theorem \ref{thm:AAAI2023:excess_risk_bound}}

    \begin{proof}
        We first prove the following equalities,
        \begin{align*}
            \mathbb{E}_{r}\mathop{\mathbb{E}}_{({\bf x},y)\sim\mathcal{D}}[\ell(\bar{f}(\mathbf{x}),y)]
            =&\frac{1}{T}\sum^T_{r=1}\mathop{\mathbb{E}}_{({\bf x},y)\sim\mathcal{D}}
            [\ell(f_{r}(\mathbf{x}),y)]\\
            =&\frac{1}{T}\sum^T_{r=1}
            \mathop{\mathbb{E}}_{({\bf x},y)\sim\mathcal{D}}[\ell(f_{r}(\mathbf{x}_{r}),y_{r})],
        \end{align*}
        where we use the fact $(\mathbf{x}_r,y_r)$
        and $(\mathbf{x},y)$ follow $\mathcal{D}$ and are independent.
        We further have
        \begin{align*}
            &T\mathbb{E}_{r}\mathop{\mathbb{E}}_{({\bf x},y)\sim\mathcal{D}}
            [\ell(\bar{f}(\mathbf{x}),y)]\\
            =&\mathop{\mathbb{E}}_{({\bf x},y)\sim\mathcal{D}}
            \left[\sum^T_{r=1}\ell(f_{r}(\mathbf{x}_{r}),y_{r})\right]\\
            \leq&\mathbb{E}_{({\bf x},y)\sim\mathcal{D}}\left[L_T(f)+\mathrm{Reg}(f)\right]\\
            =&T\mathcal{E}(f)+6U\frac{\sqrt{\mathcal{A}_T}}{T}+\frac{8U}{T}.
        \end{align*}
        Recalling the notation $\mathcal{E}(\bar{f})
        :=\mathbb{E}_{({\bf x},y)\sim\mathcal{D}}[\ell(\bar{f}(\mathbf{x}),y)]$.
        Dividing by $T$ on both sides and taking expectation w.r.t. $\mathcal{I}_T$ gives
        $$
            \mathbb{E}_{\mathcal{I}_T,r}[\mathcal{E}(\bar{f})]
            \leq \mathcal{E}(f)
            +6U\frac{\sqrt{\mathbb{E}_{\mathcal{I}_T}[\mathcal{A}_T]}}{T}+\frac{8U}{T}.
        $$
        Solving for $\mathbb{E}_{\mathbb{S},r}[\mathcal{E}(f_{\bar{{\bf w}}})]$
        and omitting the lower order terms concludes the proof.\\
        Finally,
        we analyze the computational complexity.
        Using Corollary 3 in \cite{Sun2012On},
        if the eigenvalues of $\mathcal{C}$ decay exponentially,
        i.e., $\lambda_i\leq R_0r^{i}$,
        then with probability at least $1-\delta$,
        $k=O(\frac{\ln{T}}{\ln{(r\delta^{-1})}})$.
        Thus the space complexity and per-round time complexity
        is $O(d\frac{\ln{T}}{\ln{(r\delta^{-1})}}+(\frac{\ln{T}}{\ln{(r\delta^{-1})}})^2)$.
    \end{proof}

\section{Computing}

    In this section,
    we show how to compute some keys values.

\subsection{Computing $\alpha_t$}

    Recalling that
    $$
        \left(\min_{{\bm \beta}\in\mathbb{R}^{\vert S_t\vert}}
        \left\Vert {\bm \Phi}_{S_t}{\bm \beta}-\kappa(\mathbf{x}_t,\cdot)\right\Vert^2_{\mathcal{H}}\right)
        =:\alpha_t.
    $$
    The optimal solution is
    $$
        {\bm \beta}^\ast_t = {\bf K}^{-1}_{S_t}{\bm \Phi}^\top_{S_t}\kappa(\mathbf{x}_t,\cdot)
        ={\bf K}^{-1}_{S_t}k_{S_t}.
    $$
    Substituting into the definition of $\alpha_t$ gives
    \begin{align*}
        \alpha_t
        =&\kappa(\mathbf{x}_t,\mathbf{x}_t)
        -2({\bm \beta}^\ast_t)^\top k_{S_t}
        +({\bm \beta}^\ast_t)^\top {\bm \Phi}^\top_{S_t}{\bm \Phi}_{S_t}{\bm \beta}^\ast_t\\
        =&\kappa(\mathbf{x}_t,\mathbf{x}_t)-k^\top_{S_t}{\bf K}^{-1}_{S_t}k_{S_t}.
    \end{align*}

\subsection{Computing ${\bf K}^{-1}_{S_{t+1}}$}

    There are two approaches to compute ${\bf K}^{-1}_{S_{t+1}}$ incrementally
    in time $O(d\vert S_t\vert+\vert S_t\vert^2)$.

    The first approach uses the formula of inverse of block matrix.
    \begin{align*}
        \left[
         \begin{array}{cc}
           {\bf A} & {\bf B} \\
           {\bf C} & {\bf D} \\
         \end{array}
        \right]^{-1}
        =\left[
         \begin{array}{cc}
           {\bf A}^{-1}+ {\bf A}^{-1}{\bf B}{\bf E}{\bf C}{\bf A}^{-1}& -{\bf A}^{-1}{\bf B}{\bf E} \\
           -{\bf E}{\bf C}{\bf A}^{-1} & {\bf E} \\
         \end{array}
        \right]
    \end{align*}
    where
    $$
        {\bf E} =({\bf D}-{\bf C}{\bf A}^{-1}{\bf B})^{-1}.
    $$
    At any round $t$,
    if we add $({\bf x}_t,y_t)$ into $S_t$,
    then $S_{t+1}=S_t\cup\{({\bf x}_t,y_t)\}$.
    Let ${\bf A}={\bf K}_{S_{t}}$,
    ${\bf C}={\bm \Phi}^\top_{S_t}\kappa(\mathbf{x}_t,\cdot)$,
    ${\bf B}={\bf C}^{\top}$
    and
    ${\bf D}=\kappa(\mathbf{x}_t,\mathbf{x}_t)$.
    Then we have
    $$
        {\bf K}^{-1}_{S_{t+1}}=\left[
         \begin{array}{cc}
           {\bf A} & {\bf B} \\
           {\bf C} & {\bf D} \\
         \end{array}
        \right]^{-1}.
    $$
    It is clear that the time complexity is $O(d\vert S_t\vert+\vert S_t\vert^2)$.

    The second approach which is slightly different the first approach,
    follows the Projectron algorithm \cite{Orabona2008The}.
    Denote by $S_{t+1}=S_t\cup\{({\bf x}_t,y_t)\}$.
    We first construct $\tilde{{\bf K}}^{-1}_{S_{t+1}}$ as follows
    $$
        \tilde{{\bf K}}^{-1}_{S_{t+1}}=\left(
                                         \begin{array}{cc}
                                           {\bf K}^{-1}_{S_t} & {\bf 0}_{\vert S_t\vert} \\
                                            {\bf 0}^\top_{\vert S_t\vert} & 0 \\
                                         \end{array}
                                       \right)
    $$
    where ${\bf 0}_{\vert S_t\vert}$ is a column vector of length $\vert S_t\vert$.
    Then we compute ${\bf K}^{-1}_{S_{t+1}}$ as follows
    $$
        {\bf K}^{-1}_{S_{t+1}}
        =\tilde{{\bf K}}^{-1}_{S_{t+1}}
        +\frac{1}{\alpha_t}\cdot
        \left(
         \begin{array}{c}
           {\bm \beta}^\ast_t \\
           -1 \\
         \end{array}
       \right)^\top
       \left(
         \begin{array}{cc}
           ({\bm \beta}^\ast_t)^\top & -1 \\
         \end{array}
       \right)
        ,
    $$
    where
    $
        {\bm \beta}^\ast_t = {\bf K}^{-1}_{S_t}{\bm \Phi}^\top_{S_t}\kappa(\mathbf{x}_t,\cdot).
    $
    The time complexity is $O(d\vert S_t\vert+\vert S_t\vert^2)$.

    In our experiments,
    we use the second approach.

\subsection{Computing $f_t(\mathbf{x}_t)$}

    Recalling the definition of $f_t$,
    we have
    \begin{align*}
        f_t({\bf x}_t)
        =&\langle f'_{t-1}-\lambda_t\bar{\nabla}_t,\kappa({\bf x}_t,\cdot)\rangle.
    \end{align*}
    For $t\leq M$,
    we have $\bar{\nabla}_t=\frac{1}{t-1}\sum^{t-1}_{\tau=1}-y_{\tau}\kappa({\bf x}_{\tau},\cdot)$ and
    \begin{align*}
        f_t({\bf x}_t)
        =f'_{t-1}({\bf x}_t)+\lambda_t\frac{1}{t-1}\sum^{t-1}_{\tau=1}y_{\tau}\kappa({\bf x}_{\tau},{\bf x}_t).
    \end{align*}
    For $t>M$,
    we have $\bar{\nabla}_t=\frac{1}{M}\sum^{M}_{r=1}-y_{t-r}\kappa({\bf x}_{t-r},\cdot)$ and
    \begin{align*}
        f_t({\bf x}_t)
        =f'_{t-1}({\bf x}_t)+\lambda_t\frac{1}{M}\sum^{M}_{r=1}
        y_{t-r}\kappa({\bf x}_{t-r},{\bf x}_t).
    \end{align*}

%

\subsection{Computing $\Vert\tilde{\nabla}_t-\bar{\nabla}_t\Vert^2_{\mathcal{H}}
-\Vert\bar{\nabla}_t\Vert^2_{\mathcal{H}}$}

    We first consider $t>M$.\\
    If $\tilde{\nabla}_t=-y_t\kappa(\mathbf{x}_t,\cdot)$, then
    \begin{align*}
        &\Vert\tilde{\nabla}_t-\bar{\nabla}_t\Vert^2_{\mathcal{H}}
    -\Vert\bar{\nabla}_t\Vert^2_{\mathcal{H}}\\
    =&\Vert\tilde{\nabla}_t\Vert^2_{\mathcal{H}}-2\langle \tilde{\nabla}_t,\bar{\nabla}_t\rangle\\
    =&\kappa(\mathbf{x}_t,\mathbf{x}_t)
    -2y_t\cdot\frac{1}{M}\sum^M_{r=1}y_{t-r}\kappa({\bf x}_t,{\bf x}_{t-r}).
    \end{align*}
    If $\tilde{\nabla}_t=-y_t\Phi_{S_t}{\bm \beta}^{\ast}_t$, then
    \begin{align*}
    &\Vert\tilde{\nabla}_t-\bar{\nabla}_t\Vert^2_{\mathcal{H}}-\Vert\bar{\nabla}_t\Vert^2_{\mathcal{H}}\\
    =&({\bm \beta}^{\ast}_t)^\top{\bf K}_{S_t}{\bm \beta}^{\ast}_t
    -2y_t\cdot\frac{1}{M}\sum^M_{r=1}\sum_{{\bf x}_i\in S_t}\beta_{t,i}y_{t-r}\kappa({\bf x}_i,{\bf x}_{t-r}).
    \end{align*}
    Next we consider $2\leq t\leq M$.\\
    In this case, $\bar{\nabla}_t=\frac{1}{t-1}\sum^{t-1}_{\tau=1}-y_{\tau}\kappa({\bf x}_{\tau},\cdot)$.
    If $\tilde{\nabla}_t=-y_t\kappa(\mathbf{x}_t,\cdot)$, then
    \begin{align*}
    \Vert\tilde{\nabla}_t-\bar{\nabla}_t\Vert^2_{\mathcal{H}}
    -\Vert\bar{\nabla}_t\Vert^2_{\mathcal{H}}
    =&\kappa(\mathbf{x}_t,\mathbf{x}_t)
    -\frac{2y_t}{t-1}\sum^{t-1}_{\tau=1}y_{\tau}\kappa({\bf x}_t,{\bf x}_{\tau}).
    \end{align*}
    If $\tilde{\nabla}_t=-y_t\Phi_{S_t}{\bm \beta}^{\ast}_t$, then
    \begin{align*}
    &\Vert\tilde{\nabla}_t-\bar{\nabla}_t\Vert^2_{\mathcal{H}}-\Vert\bar{\nabla}_t\Vert^2_{\mathcal{H}}\\
    =&({\bm \beta}^{\ast}_t)^\top{\bf K}_{S_t}{\bm \beta}^{\ast}_t
    -2y_t\cdot\frac{1}{t-1}\sum^{t-1}_{\tau=1}\sum_{{\bf x}_i\in S_t}\beta_{t,i}y_{\tau}
    \kappa({\bf x}_i,{\bf x}_{\tau}).
    \end{align*}

\subsection{Computing $\Vert f'_t\Vert_{\mathcal{H}}$}

    We can compute $\Vert f'_t\Vert_{\mathcal{H}}$ incrementally.
    Recalling that
    $$
        f'_t = f'_{t-1}-\lambda_t\tilde{\nabla}_t.
    $$
    Thus
    $$
        \Vert f'_t\Vert^2_{\mathcal{H}}
        =\Vert f'_{t-1}\Vert^2_{\mathcal{H}}+\lambda_t\Vert \tilde{\nabla}_t\Vert^2_{\mathcal{H}}
        -2\lambda_t\langle f'_{t-1},\tilde{\nabla}_t\rangle.
    $$
    If $\tilde{\nabla}_t=-y_t\kappa(\mathbf{x}_t,\cdot)$, then
    $$
        \Vert f'_t\Vert^2_{\mathcal{H}}
        =\Vert f'_{t-1}\Vert^2_{\mathcal{H}}+\lambda_t\kappa(\mathbf{x}_t,\mathbf{x}_t)
        +2\lambda_t y_tf'_{t-1}(\mathbf{x}_t).
    $$
    If $\tilde{\nabla}_t=-y_t\Phi_{S_t}{\bm \beta}^{\ast}_t$,
    then
    \begin{align*}
        &\Vert f'_t\Vert^2_{\mathcal{H}}\\
        =&\Vert f'_{t-1}\Vert^2_{\mathcal{H}}+\lambda_t({\bm \beta}^{\ast}_t)^\top{\bf K}_{S_t}({\bm \beta}^{\ast}_t)
        +2\lambda_t y_t\sum_{{\bf x}_i\in S_t}\beta^\ast_{t,i}f'_{t-1}(\mathbf{x}_i).
    \end{align*}

\section{More Experiments}

    In this section,
    we give the experimental results on more datasets,
    including \textit{minist12:12,700, a9a:48,842, SUSY:50,000, ijcnn1: 141,691}.
    We construct the \textit{minist12} dataset by extracting instances with label 1
    and instances with label 2 from original \textit{minist} dataset.
    Similarly,
    we construct the \textit{SUSY} dataset by uniformly selecting 5,0000 instances
    from original \textit{SUSY} dataset.

    Table \ref{tab:Supp_AAAI2023:POMDR:AMR} shows the average mistake ratio and
    average running time.
    As a whole,
    our algorithm enjoys the best prediction performance on all
    datasets except for the \textit{ijcnn1} dataset.
    The running time of our algorithm
    is also comparable with the baseline algorithms.
    The experimental results verify the first goal \textbf{G}~2.
    \begin{table}[!t]
      \centering
      \begin{tabular}{l|c|r|r}
      \toprule
        \multirow{2}{*}{Algorithm} & \multirow{2}{*}{$B$}
        & \multicolumn{2}{c}{minist12,~$\sigma=4$, $\zeta=2/3$}\\
        \cline{3-4}&&AMR (\%)&Time (s) \\
      \toprule
        FOGD     & 400 & 0.94 $\pm$ 0.08 & 4.85 $\pm$ 0.04 \\
        NOGD     & 400 & 1.89 $\pm$ 0.22 & 15.50 $\pm$ 0.09 \\
        SkeGD    & 400 & \textbf{0.60} $\pm$ \textbf{0.04} & 19.26 $\pm$ 0.81 \\
        $\mathrm{B}\mathrm{(AO)}_2\mathrm{KS}$   & 400 & 1.41 $\pm$ 0.12 & 8.34 $\pm$ 0.72 \\
        POMDR    & 400 & \textbf{0.53} $\pm$ \textbf{0.04} & 9.85 $\pm$ 0.34 \\
      \toprule
        \multirow{2}{*}{Algorithm} & \multirow{2}{*}{$B$}
        & \multicolumn{2}{c}{a9a,~$\sigma=4$, $\zeta=1/2$}\\
        \cline{3-4}&&AMR (\%)&Time (s) \\
      \toprule
        FOGD     & 400 & \textbf{16.39} $\pm$ \textbf{0.07} & 4.47 $\pm$ 0.03 \\
        NOGD     & 400 & 16.53 $\pm$ 0.10 & 9.95 $\pm$ 0.18 \\
        SkeGD    & 400 & 17.31 $\pm$ 0.21 & 4.88 $\pm$ 0.03 \\
        $\mathrm{B}\mathrm{(AO)}_2\mathrm{KS}$   & 400 & 20.65 $\pm$ 0.26 & 6.16 $\pm$ 0.08 \\
        POMDR    & 400 & \textbf{16.25} $\pm$ \textbf{0.09} & 9.53 $\pm$ 0.05 \\
      \toprule
        \multirow{2}{*}{Algorithm} & \multirow{2}{*}{$B$}
        & \multicolumn{2}{c}{ijcnn1,~$\sigma=4$, $\zeta=2/3$}\\
        \cline{3-4}&&AMR (\%)&Time (s) \\
      \toprule
        FOGD     & 400 & \textbf{8.51} $\pm$ \textbf{0.06} & 6.14 $\pm$ 0.14 \\
        NOGD     & 400 & \textbf{8.53} $\pm$ \textbf{0.03} & 12.63 $\pm$ 0.31 \\
        SkeGD    & 400 & 9.76 $\pm$ 0.05 & 7.58 $\pm$ 0.26 \\
        $\mathrm{B}\mathrm{(AO)}_2\mathrm{KS}$   & 400 & 9.58 $\pm$ 0.00 & 7.12 $\pm$ 0.35 \\
        POMDR    & 400 & 8.69 $\pm$ 0.07 & 8.37 $\pm$ 0.08 \\
      \toprule
        \multirow{2}{*}{Algorithm} & \multirow{2}{*}{$B$}
        & \multicolumn{2}{c}{SUSY,~$\sigma=4$, $\zeta=2/3$}\\
        \cline{3-4}&&AMR (\%)&Time (s) \\
      \toprule
        FOGD     & 400 & 22.16 $\pm$ 0.12 & 2.21 $\pm$ 0.09 \\
        NOGD     & 400 & 22.31 $\pm$ 0.13 & 4.79 $\pm$ 0.28 \\
        SkeGD    & 400 & 27.47 $\pm$ 0.88 & 2.46 $\pm$ 0.03 \\
        $\mathrm{B}\mathrm{(AO)}_2\mathrm{KS}$   & 400 & 30.71 $\pm$ 0.20 & 3.29 $\pm$ 0.34 \\
        POMDR    & 400 & \textbf{21.81} $\pm$ \textbf{0.07} & 3.17 $\pm$ 0.34 \\
      \bottomrule
      \end{tabular}
      \caption{Average online mistake ratio (AMR) of algorithms.}
      \label{tab:Supp_AAAI2023:POMDR:AMR}
    \end{table}

    Next we verify \textbf{G} 1.
    Table \ref{tab:supp_AAAI2023:experiments_xi_T} record
    the value of $\mathcal{A}_T$ and $\bar{t}$ under different kernel parameter, $\sigma$.
    If we choose a suitable $\sigma$,
    then the value of $\mathcal{A}_T\ll T$ is possible,
    such as the \textit{minist12} dataset.
    For the \textit{a9a} dataset,
    the value of $\mathcal{A}_T$ is not much smaller than $T$.
    In this case,
    our data-dependent regret bound still matches the worst regret bound.
    i.e., $O(\sqrt{T})$.
    The last row of Table \ref{tab:supp_AAAI2023:experiments_xi_T}
    gives the value of $\bar{t}$.
    We can also find that if we choose a suitable $\sigma$,
    then $\bar{t}=+\infty$.
    In this case,
    the size of budget is smaller than $B_0=\lceil 15\ln{T}\rceil$.
    The results verify \textbf{G}~1.

    \begin{table}[!t]
      \centering
      \begin{tabular}{l|rr|rr|rr}
      \toprule
       -& \multicolumn{2}{c|}{minist12}  &\multicolumn{2}{c|}{a9a} &\multicolumn{2}{c}{SUSY}\\
       \hline
      $\sigma$& 0.5 & 4 & 4 & 32 & 4 & 32 \\
      \toprule
      AMR  &2.43   & \textbf{0.53}  &\textbf{16.25}     &21.17  &\textbf{21.81}    &29.45   \\
      $\mathcal{A}_T$&12698   & 376  &15405    &26801& 22378    &39468   \\
      $\bar{t}$      &142  & $3228$  &529    &$+\infty$& 361    &$+\infty$   \\
      \bottomrule
      \end{tabular}
      \caption{The value of $\mathcal{A}_T$ and $\bar{t}$.}
      \label{tab:supp_AAAI2023:experiments_xi_T}
    \end{table}

\end{document}